\documentclass[twoside,11pt]{article}

\usepackage{lineno}
\usepackage{amsmath}

\DeclareMathOperator*{\argmin}{arg\,min}
\DeclareMathOperator{\diag}{diag}
\DeclareMathOperator{\trace}{tr}
\usepackage[caption=false,font=footnotesize]{subfig}
\usepackage{url}
\usepackage{pifont}
\usepackage{algorithm}
\usepackage{algorithmic}

\usepackage{booktabs}
\usepackage{siunitx}
\usepackage{enumitem}

\usepackage{jmlr2e}

\newcommand{\dataset}[1]{\texttt{#1}}
\newcommand{\defeq}{\stackrel{\text{\tiny def}}{=}}

\newcommand{\order}{\mathcal{O}} 
\newcommand{\x}{\mathbf{x}}
\renewcommand{\xi}{{\x}_{i}}

\newcommand{\z}{\mathbf{z}}
\newcommand{\X}{\mathbf{X}}
\newcommand{\QQ}{\mathbf{Q}}
\newcommand{\RR}{\mathbf{R}}

\newcommand{\K}{\mathbf{K}}
\newcommand{\G}{\mathbf{G}}
\newcommand{\y}{\mathbf{y}}
\newcommand{\U}{\mathbf{U}}

\newcommand{\eye}{\mathbf{I}}
\newcommand{\R}{\mathbb{R}}    
\newcommand{\LL}{\mathbf{L}}
\newcommand{\SSS}{\mathcal{S}}

\newcommand{\Z}{\mathbf{Z}}
\newcommand{\muu}{\boldsymbol{\mu}}
\newcommand{\Lam}{\boldsymbol{\Lambda}}

\newcommand{\CC}{\mathbf{C}}
\newcommand{\WW}{\mathbf{W}}

\newcommand{\UU}{\mathbf{U}}
\newcommand{\VV}{\mathbf{V}}
\newcommand{\SIGMA}{\boldsymbol{\Sigma}}
\newcommand{\HH}{\mathbf{H}}

\firstpageno{1}

\begin{document}

\title{Randomized Clustered Nystr{\"o}m for Large-Scale Kernel Machines}

\author{\name Farhad Pourkamali-Anaraki \email farhad.pourkamali@colorado.edu \\
       \addr Department of Electrical, Computer, and Energy Engineering\\
       University of Colorado Boulder\\
       Boulder, CO 80309-0425, USA
       \AND
       \name Stephen Becker \email stephen.becker@colorado.edu \\
       \addr Department of Applied Mathematics\\
       University of Colorado Boulder\\
       Boulder, CO 80309-0526, USA}

\editor{}

\maketitle

\begin{abstract}
The Nystr\"om method has been popular for generating the low-rank approximation of kernel matrices that arise in many machine learning problems. The approximation quality of the Nystr\"om method depends crucially on the number of selected landmark points and the selection procedure. In this paper, we present a novel algorithm to compute the optimal Nystr\"om low-approximation when the number of landmark points exceed the target rank. Moreover, we introduce a randomized algorithm for generating landmark points that is scalable to large-scale data sets. The proposed method performs K-means clustering on low-dimensional random projections of a data set and, thus, leads to significant savings for high-dimensional data sets.
Our theoretical results characterize the tradeoffs between the accuracy and efficiency of our proposed method. Extensive experiments demonstrate the competitive performance as well as the efficiency of our proposed method.
\end{abstract}

\begin{keywords}
  Kernel methods, Nystr\"om method, Low-rank approximation, Random projections, Large-scale learning
\end{keywords}

\section{Introduction}
\label{sec:introduction}
Kernel machines have been widely used in various machine learning problems such as classification, clustering, and regression. In kernel-based learning, the input data points are mapped to a high-dimensional feature space and the pairwise inner products in the lifted space are computed and stored in a positive semidefinite kernel matrix $\K$. The lifted representation may lead to better performance of the learning problem, but a drawback is the need to store and manipulate a large kernel matrix of size $n\times n$, where $n$ is the size of data set. Thus a kernel machine has quadratic space complexity and quadratic or cubic computational complexity (depending on the specific type of machine).

One promising strategy for reducing these costs consists of a low-rank approximation of the kernel matrix $\K\approx\LL\LL^T$, where $\LL\in\R^{n\times r}$ for a target rank $r<n$. Such low-rank approximations can be used to reduce the memory and computation cost by trading-off accuracy for scalability. For this reason, much research has focused on efficient algorithms for computing low-rank approximations, e.g.,~\citep{fine2001efficient,bach2002kernel,Jordan_Kernel_Approx,Martinson_SVD}. The Nystr\"om method is probably one of the most well-studied and successful methods that has been used to scale up several kernel methods~\citep{kumar2009ensemble,sun2015review}. 
The Nystr\"om method works by selecting a small set of bases referred to as ``landmark points'' and computing the  kernel similarities between the input data points and landmark points. Therefore, the performance of the Nystr\"om method depends crucially on the number of selected landmark points as well as the procedure according to which these landmark points are selected. 

The original Nystr\"om method,
first introduced to the kernel machine setting by~\citet{Nystrom2001}, proposed to select landmark points uniformly at random from the set of input data points. More recently, several other probabilistic strategies have been proposed to provide informative landmark points in the Nystr\"om method, including sampling with weights proportional to column norms~\citep{drineas2006fast}, diagonal entries~\citep{Nystrom_Kernel_Approx}, and leverage scores~\citep{gittens2013revisiting}. \citet{zhang2010clusteredNys} proposed a non-probabilistic technique for generating landmark points using centroids resulting from K-means clustering on the input data points. The proposed ``Clustered Nystr\"om method'' shows the Nystr\"om approximation error is related to the encoding power of landmark points in summarizing data and it provides improved accuracy over other sampling methods such as uniform and column-norm sampling~\citep{kumar2012sampling}. However, the main drawback of this method is the high memory and computational complexity associated with performing K-means clustering on high-dimensional large-scale data sets.

The aim of this paper is to improve the accuracy and efficiency of the Nystr\"om method in two directions. We present a novel algorithm to compute the optimal rank-$r$ approximation in the Nystr\"om method when the number of landmark points exceed the rank parameter $r$. In fact, our proposed method can be used within all landmark selection procedures to compute the best rank-$r$ approximation achievable by a chosen set of landmark points. 
Moreover, we present an efficient method for landmark selection which provides a tunable tradeoff between the accuracy of low-rank approximations and memory/computation requirements. Our proposed ``Randomized Clustered Nystr\"om method'' generates a set of landmark points based on low-dimensional random projections of the input data points~\citep{DataBaseFriendlyRandomProjection}.

In more detail, our main contributions are threefold.
\begin{itemize}[leftmargin=*]
	\item It is common to select more landmark points than the target rank $r$ to obtain high quality Nystr\"om low-rank approximations. In Section~\ref{sec:nys-eig}, we present a novel algorithm with theoretical analysis for computing the optimal rank-$r$ approximation when the number of landmark points exceed the target rank $r$. Thus, our proposed method, called ``Nystr\"om via QR Decomposition,'' can be used with any landmark selection algorithm to find the best rank-$r$ approximation for a given set of landmark points. We also provide intuitive and real-world examples to show the superior performance and efficiency of our method in Section~\ref{sec:nys-eig}.
	\item Second, we present a random-projection-type landmark selection algorithm which easily scales to large-scale high-dimensional data sets. Our proposed ``Randomized Clustered Nystr\"om method'' presented in Section~\ref{sec:random-clustered-nys} performs the K-means clustering algorithm on the random projections of input data points and it requires only two passes over the original data set. Thus our method leads to significant memory and computation savings in comparison with the Clustered Nystr\"om method.
	Moreover, our theoretical results (Theorem~\ref{thm:randomized-clustered-nys}) show that the proposed method produces low-rank approximations with little loss in accuracy compared to Clustered Nystr\"om with high probability.   
	\item Third, we present extensive numerical experiments comparing our Randomized Clustered Nystr\"om method with a few other sampling methods on two tasks: (1) low-rank approximation of kernel matrices and (2) kernel ridge regression. In Section~\ref{sec:experiments}, we consider six data sets from the LIBSVM archive~\citep{CC01a} with dimensionality up to $p=150,\!360$.
\end{itemize}

\section{Notation and Preliminaries}\label{sec:notation}
We denote column vectors with lower-case bold letters and matrices with upper-case bold letters. $\eye_{n\times n}$ is the identity matrix of size $n\times n$; $\mathbf{0}_{m\times n}$ is the $m\times n$ matrix of zeros. For a vector $\x\in\R^p$, let $\|\x\|_2$ denote the Euclidean norm, and $\diag(\x)$ represents a diagonal matrix with the elements of $\x$ on the main diagonal. The Frobenius norm for a matrix $\mathbf{A}\in\R^{n\times m}$ is defined as $\|\mathbf{A}\|_F=(\sum_{i=1}^n\sum_{j=1}^{m}A_{ij}^2)^{1/2}=(\trace(\mathbf{A}^T\mathbf{A}))^{1/2}$, where $A_{ij}$ represents the $(i,j)$-th entry of $\mathbf{A}$, $\mathbf{A}^T$ is the transpose of $\mathbf{A}$, and $\trace(\cdot)$ is the trace operator.

Let $\K\in\R^{n\times n}$ be a symmetric positive semidefinite (SPSD) matrix with
$\text{rank}(\K)=\rho\leq n$. The singular value decomposition (SVD) or eigenvalue decomposition of $\K$ can be written as $\K=\UU\Lam\UU^T$, where $\UU\in\R^{n\times \rho}$ contains the orthonormal eigenvectors, i.e., $\UU^T\UU=\eye_{\rho\times \rho}$, and  $\Lam=\diag\left([\lambda_1(\K),\ldots,\lambda_\rho(\K)]\right)\in\R^{\rho \times \rho}$ is a diagonal matrix which contains the eigenvalues of $\K$ in descending order, i.e., $\lambda_1(\K)\geq\ldots\geq\lambda_\rho(\K)$. The matrices $\UU$ and $\Lam$ can be decomposed for a target rank $r$ ($r\leq\rho$):
\begin{eqnarray}
\K&=&\left(\begin{array}{cc}
\UU_r & \UU_{\rho-r}\end{array}\right)\left(\begin{array}{cc}
\Lam_r & \mathbf{0}_{r\times (\rho-r)}\\
\mathbf{0}_{(\rho-r)\times r} & \Lam_{\rho-r}
\end{array}\right)\left(\begin{array}{c}
\UU_r^T\\
\UU_{\rho-r}^T
\end{array}\right)\nonumber\\
&=& \UU_r\Lam_r\UU_r^T+\UU_{\rho-r}\Lam_{\rho-r}\UU_{\rho-r}^T,\label{eq:eig-decomp}
\end{eqnarray}
where $\Lam_r\in\R^{r\times r}$ contains the $r$ leading eigenvalues and the columns of $\UU_r\in\R^{n\times r}$ span the top $r$-dimensional eigenspace, and $\Lam_{\rho-r}$ and $\UU_{\rho-r}$ contain the remaining $(\rho-r)$ eigenvalues and eigenvectors. It is well-known that $\K_{(r)}=\UU_r\Lam_r\UU_r^T$ is the ``best rank-$r$ approximation'' to $\K$ in the sense that $\K_{(r)}$ minimizes $\|\K-\mathbf{A}\|_F$ over all matrices $\mathbf{A}\in\R^{n\times n}$ of rank at most $r$ \citep{eckart1936approximation} and we have $\|\K-\K_{(r)}\|_F=(\sum_{i=r+1}^{\rho}\lambda_i(\K)^2)^{1/2}$.
If $\lambda_r(\K) = \lambda_{r+1}(\K)$, then $\K_{(r)}$ is not unique, so we write $\K_{(r)}$ to mean any matrix satisfying Equation \ref{eq:eig-decomp}.
%
The pseudo-inverse of  $\K$ 
can be obtained from the SVD or eigenvalue decomposition
as
$\K^\dagger=\UU_\rho\Lam_\rho^{-1}\UU_\rho^T$.
When $\K$ is full rank, we have $\K^\dagger=\K^{-1}$.

Another matrix factorization technique that we use in this paper is the QR decomposition. An $n\times m$ matrix $\mathbf{A}$, with $n\geq m$, can be decomposed as a product of two matrices $\mathbf{A}=\QQ\RR$,
where $\QQ\in\R^{n\times m}$ has $m$ orthonormal columns, i.e., $\QQ^T\QQ=\eye_{m\times m}$, and $\RR\in\R^{m\times m}$ is an upper triangular matrix.
Sometimes this is called the \emph{thin} QR decomposition, to distinguish it from a \emph{full} QR decomposition which finds $\QQ\in\R^{n\times n}$ and zero-pads $\RR$ accordingly.

\section{Background and Related Work}\label{sec:related-work}
Kernel methods have been successfully applied to  a variety of machine learning problems such as classification and regression. Well-known examples include support vector machines (SVM)~\citep{VapnikSVM}, kernel principal component analysis (KPCA)~\citep{scholkopf1998nonlinear}, kernel ridge regression~\citep{shawe2004kernel}, kernel clustering~\citep{girolami2002mercer}, and kernel dictionary learning~\citep{Nonlinear_kernel_DL}. The main idea
behind kernel-based learning is to map the input data points into a feature space, where all pairwise inner products of the mapped data points can be computed via a nonlinear kernel function that satisfies  Mercer's condition~\citep{aronszajn1950theory,LearningWithKernels}. Thus, kernel methods allow one to use linear algorithms in the higher (or infinite) dimensional feature space which correspond
to nonlinear algorithms in the original space. For this reason, kernel machines have received much attention as an effective tool to tackle problems with complex and nonlinear structures.

Let $\X=\left[\x_1,\ldots,\x_n\right]\in\R^{p\times n}$ be a data matrix that contains $n$ data points in $\R^p$ as its columns. The inner products in feature space are calculated using a ``kernel function'' $\kappa\left(\cdot,\cdot\right)$ defined on the original space:
\begin{eqnarray*}
K_{ij}:=\kappa\left(\x_i,\x_j\right)=
\langle\Phi(\x_i),\Phi(\x_j)\rangle,\;\;\forall i,j=1,\ldots,n,\label{eq:kernel-feature-space}
\end{eqnarray*}
where $\Phi:\x\mapsto\Phi(\x)$ is the kernel-induced feature map. All pairwise inner products of the $n$ mapped data points are stored in the so-called ``kernel matrix'' $\K\in\R^{n\times n}$, where the $(i,j)$-th entry is $K_{ij}$. Two well-known examples of kernel functions that lead to symmetric positive semidefinite (SPSD) kernel matrices are Gaussian and polynomial kernel functions. The former takes the form $\kappa\left(\x_i,\x_j\right)=\exp\left(-\|\x_i-\x_j\|_2^2/c\right)$ and the polynomial kernel is of the form $\kappa\left(\x_i,\x_j\right)=\left(\langle\x_i,\x_j\rangle+c\right)^d$, where $c\in\R^+$ and $d\in\mathbb{N}$ are the parameters~\citep{van2012kernel,anaraki2013kernel}. Moreover, combinations of multiple kernels can be constructed to tackle problems with complex and heterogeneous
data sources~\citep{bach2004multiple,gonen2011multiple,liu2016kernelized}. 

Despite the simplicity of kernel machines in nonlinear representation of data, one prominent problem is the calculation, storage, and manipulation of the kernel matrix for large-scale data sets. The cost to form $\K$ using standard kernel functions is $\order(pn^2)$ and it takes $\order(n^2)$ memory to store the full kernel matrix. Thus, both memory and computation cost scale as the square of the number of data points. Moreover, subsequent processing of the kernel matrix within the learning process is computationally quite expensive. For example, algorithms such as KPCA and kernel dictionary learning compute the eigenvalue decomposition of the kernel matrix, where the standard techniques take $\order(n^3)$ time and multiple passes over $\K$ will be required. In other kernel-based learning methods such as kernel ridge regression, the inverse of the kernel matrix $\left(\K+\lambda \eye_{n\times n}\right)^{-1}$, where $\lambda>0$ is a regularization parameter, must be computed which requires  $\order(n^3)$ time~\citep{cortes2010impact,el2014fast}. Thus, large-scale data sets have provided a considerable challenge to the design
of efficient kernel-based learning algorithms~\citep{BigDataSigMagazine,hsieh2014fast}.

A well-studied approach to reduce the memory and computation burden associated with kernel machines is to use a low-rank approximation of kernel matrices. This approach utilizes the decaying spectra of kernel matrices and the best rank-$r$ approximation $\K_{(r)}=\U_r \Lam_r \U_r^T$ is computed, cf.~Equation \ref{eq:eig-decomp}. Since $\K$ is SPSD, the eigenvalue decomposition can be used to express a low-rank approximation in the form of:
\begin{eqnarray*}
\K_{(r)} = \LL\LL^T,\;\;\LL=\U_r\Lam_r^{1/2}\in\R^{n\times r}.\label{eq:K-low-rank}
\end{eqnarray*}

The benefits of this low-rank approximation are twofold. First, it takes $\order(nr)$ to store the matrix $\LL$ which is only linear in the data set
size $n$. The reduction of memory requirements from quadratic to linear results in significant memory savings. Second, the low-rank approximation leads to substantial computational savings within the learning process. For example, the following matrix inversion arising in algorithms such as kernel ridge regression can be calculated using the Sherman-Morrison-Woodbury formula:   
\begin{eqnarray}
\left(\K+\lambda\eye_{n\times n}\right)^{-1}&\approx&\left(\K_{(r)}+\lambda\eye_{n\times n}\right)^{-1}\nonumber\\
& = & \left(\LL\LL^T+\lambda\eye_{n\times n}\right)^{-1}\nonumber\\
& = & \lambda^{-1}\left(\eye_{n\times n}-\LL\left(\LL^T\LL+\lambda\eye_{n\times n}\right)^{-1}\LL^T\right).\label{eq:ridge-eq}
\end{eqnarray}
Here, we only need to invert a
much smaller matrix of size just $r\times r$. 
Thus, the computation cost is $\order(nr^2+r^3)$ to compute $\LL^T\LL$ and the matrix inversion in Equation \ref{eq:ridge-eq}. 

Another example of computation savings is the ``linearization'' of kernel methods using the low-rank approximation, where linear algorithms are applied to the rows of $\LL\in\R^{n\times r}$. In this case, the matrix $\LL$ serves as an empirical kernel map and the rows of $\LL$ are known as virtual samples. This strategy has been shown to speed up various kernel-based learning methods such as SVM, kernel dictionary learning, and kernel clustering~\citep{zhang2012scaling,golts2016linearized,pourkamali2016randomized}.     

While the low-rank approximation of kernel matrices is a promising approach to reduce the memory and computational complexity, the main bottleneck is the computation of the full kernel matrix $\K$ and the best rank-$r$ approximation $\K_{(r)}$. Standard algorithms for computing the
eigenvalue decomposition of $\K$ take $\order(n^3)$ time. Partial eigenvalue decomposition, e.g., Krylov subspace method, can be performed to find the $r$ leading eigenvalues/eigenvectors. However, these techniques require at least $r$ passes over the entire kernel matrix which is prohibitive for large dense matrices~\citep{Martinson_SVD}.

To address this problem, much recent work has focused on efficient randomized methods to compute low-rank approximations of large matrices~\citep{RandomizedAlgorithmsforMatricesAndData}. The Nystr\"om method is 
one of the few randomized approximation techniques that does not need to first compute the entire kernel matrix.
The standard Nystr\"om method was first introduced (in the context of matrix kernel approximation)  in~\citep{Nystrom2001} and is based on sampling a small subset of input data columns, after which the kernel similarities between the small subset and input data points are computed to construct a
rank-$r$ approximation. 
Section~\ref{sec:Nystrom} discusses in detail the Nystr\"om method and its extension which finds the approximate eigenvalue decomposition of the kernel matrix.

Since the sampling technique is a key aspect of the Nystr\"om method, much research has focused on selecting the most informative subset of input data to improve the approximation accuracy and thus the performance of kernel-based learning methods~\citep{kumar2012sampling}. An overview of different sampling techniques, including the Clustered Nystr\"om method, is presented in Section~\ref{sec:Nystrom-Sampling}.

\subsection{The Nystr\"om Method}\label{sec:Nystrom}
The Nystr\"om method for generating a low-rank approximation of the SPSD kernel matrix $\K\in\R^{n\times n}$ works by selecting a small set of bases referred to as ``landmark points''. For example, the simplest and most common technique to select the landmark points is based on uniform sampling without replacement from the set of all input data points~\citep{Nystrom2001}. In this section, we explain the Nystr\"om method for a given set of landmark points regardless of the sampling mechanism.

Let $\Z=[\z_1,\ldots,\z_m]\in\R^{p\times m}$ be the set of $m$ landmark points in $\R^p$. The Nystr\"om method first constructs two matrices $\CC\in\R^{n\times m}$ and $\WW\in\R^{m\times m}$, where $C_{ij}=\kappa(\x_i,\z_j)$ and $W_{ij}=\kappa(\z_i,\z_j)$. Next, it uses both $\CC$ and $\WW$ to construct a low-rank
approximation of the kernel matrix $\K$:
\begin{eqnarray*}
\G=\CC\WW^\dagger\CC^T.\label{eq:Nystrom-low-rank}
\end{eqnarray*}
For the rank-restricted case, the Nystr\"om method generates a rank-$r$ approximation of the kernel matrix, $r\leq m$, by computing the best rank-$r$ approximation of the $m\times m$ inner matrix $\WW$~\citep{kumar2012sampling,kumar2009ensemble,sun2015review,li2015large,wang2013improving}:
\begin{eqnarray}
\G_{(r)}^{nys}=\CC\WW_{(r)}^\dagger\CC^T,\label{eq:Nys-rank-r}
\end{eqnarray}
where $\WW_{(r)}^\dagger$ represents the pseudo-inverse of $\WW_{(r)}$. Thus, the eigenvalue decomposition of the matrix $\WW$ should be computed to find the top $r$ eigenvalues and corresponding eigenvectors. Let $\SIGMA_r\in\R^{r\times r}$ and $\VV_r\in\R^{m\times r}$ contain the top $r$ eigenvalues and the corresponding orthonormal eigenvectors of $\WW$, respectively. Then, the rank-$r$ approximation in Equation \ref{eq:Nys-rank-r} can be expressed as:
\begin{eqnarray}
\G_{(r)}^{nys}=\LL^{nys}\left(\LL^{nys}\right)^T,\;\; \LL^{nys}=\CC\VV_r\left(\SIGMA_r^\dagger\right)^{1/2}\in\R^{n\times r}.\label{eq:Nys-low-rank-ll}
\end{eqnarray}
The time complexity of the Nystr\"om
method to form $\LL^{nys}$ is $\order(pnm+m^2r+nmr)$, where it takes $\order(pnm)$ to construct matrices $\CC$ and $\WW$. 
Also, it takes $\order(m^2r)$ time to perform the partial eigenvalue decomposition of $\WW$ and $\order(nmr)$ represents the cost of matrix multiplication $\CC\VV_r$. Thus, for $r\leq m\ll n$, the computation cost to form the low-rank approximation of the kernel matrix, $\K\approx\LL^{nys}(\LL^{nys})^T$, is only linear in the data set size $n$. 

In practice, there exist two approaches to obtain the approximate eigenvalue decomposition of the kernel matrix in the Nystr\"om method. The first approach is based on the exact eigenvalue decomposition of $\WW$ to get the following estimates of the $r$ leading eigenvalues and eigenvectors of $\K$~\citep{kumar2012sampling}:
\begin{eqnarray}
\widehat{\UU}_{r}^{(1)}=\sqrt{\frac{m}{n}}\CC \VV_r \SIGMA_r^\dagger,\;\;\widehat{\boldsymbol{\Lambda}}_r^{(1)}=\frac{n}{m}\SIGMA_r.\label{eq:approx-eig-nys-1}
\end{eqnarray}
These estimates of eigenvalues/eigenvectors are naive since it is easy to show that the estimated eigenvectors are not guaranteed to be orthonormal, i.e., $(\widehat{\UU}_{r}^{(1)})^T\widehat{\UU}_{r}^{(1)}\neq \eye_{r\times r}$. Moreover, the factor $n/m$ in Equation \ref{eq:approx-eig-nys-1}  is used to roughly compensate for the small size of the matrix $\WW\in\R^{m\times m}$ compared to the $n\times n$ kernel matrix. Thus, the accuracy of this approach depends heavily on the data set and the selected landmark points.

The second approach provides more accurate estimates of eigenvalues and eigenvectors of $\K$ by using the low-rank approximation in Equation \ref{eq:Nys-low-rank-ll},
and in fact this approach provides the exact eigenvalue decomposition of $\G_{(r)}^{nys}$.
The first step is to find the exact eigenvalue decomposition of the $r\times r$ matrix:
\[ \big(\LL^{nys}\big)^T\LL^{nys}=\tilde{\VV}\tilde{\SIGMA}\tilde{\VV}^T,
\] 
where $\tilde{\VV},\tilde{\SIGMA}\in\R^{r\times r}$. Then, the estimates of $r$ leading eigenvalues and eigenvectors of $\K$ are obtained as follows~\citep{zhang2010clusteredNys}:
\begin{eqnarray*}
\widehat{\UU}_{r}^{(2)}=\LL^{nys}\tilde{\VV}\big(\tilde{\SIGMA}^\dagger\big)^{1/2},\;\;\widehat{\boldsymbol{\Lambda}}_r^{(2)}=\tilde{\SIGMA}.
\end{eqnarray*}
For this case, the resultant eigenvectors are orthonormal:
\newcommand{\VVV}{\widetilde{\VV}}
\begin{eqnarray*}
\big(\widehat{\UU}_{r}^{(2)}\big)^T\widehat{\UU}_{r}^{(2)}& =&\big(\tilde{\SIGMA}^\dagger\big)^{1/2}\VVV^T\big(\LL^{nys}\big)^T \LL^{nys}\VVV\big(\tilde{\SIGMA}^\dagger\big)^{1/2}\nonumber\\
& =&\big(\tilde{\SIGMA}^\dagger\big)^{1/2}\big(\VVV^T\VVV\big)\tilde{\SIGMA}\big(\VVV^T\VVV\big)\big(\tilde{\SIGMA}^\dagger\big)^{1/2}=\eye_{r\times r},
\end{eqnarray*}
where this comes from the fact that $\VVV$ contains  orthonormal eigenvectors and $(\widetilde{\SIGMA}^\dagger)^{1/2}\widetilde{\SIGMA}(\widetilde{\SIGMA}^\dagger)^{1/2}=\eye_{r\times r}$. The overall procedure to estimate the $r$ leading eigenvalues/eigenvectors based on the Nystr\"om method is summarized in Algorithm \ref{alg:StandardNys}. The time complexity of the approximate eigenvalue decomposition is $\order(nr^2+r^3)$, in addition to the cost of computing $\LL^{nys}$ mentioned earlier.

\begin{algorithm}[t]
	\caption{Standard Nystr\"om}
	\label{alg:StandardNys}
	\textbf{Input:} data set $\X$, landmark points $\Z$, kernel function $\kappa$, target rank $r$
	
	\textbf{Output:} estimates of $r$ leading eigenvectors and eigenvalues of the kernel matrix $\K\in\R^{n\times n}$: $\widehat{\UU}_r^{(2)}\in\R^{n\times r}$, $\widehat{\Lam}_r^{(2)}\in\R^{r\times r}$
	\begin{algorithmic}[1]
		\STATE  Form two matrices $\CC$ and $\WW$: $C_{ij}=\kappa(\x_i,\z_j)$, $W_{ij}=\kappa(\z_i,\z_j)$
		\STATE Compute the eigenvalue decomposition: $\WW=\VV\SIGMA\VV^T$
		\STATE Form the matrix: $\LL^{nys}=\CC\VV_r\left(\SIGMA_r^\dagger\right)^{1/2}$
		\STATE Compute the eigenvalue decomposition: $(\LL^{nys})^T\LL^{nys}=\tilde{\VV}\tilde{\SIGMA}\tilde{\VV}^T$
		\STATE $\widehat{\UU}_{r}^{(2)}=\LL^{nys}\tilde{\VV}\big(\tilde{\SIGMA}^\dagger\big)^{1/2}$ and $\widehat{\boldsymbol{\Lambda}}_r^{(2)}=\tilde{\SIGMA}$
	\end{algorithmic}
\end{algorithm}

\subsection{Sampling Techniques for the Nystr\"om Method}\label{sec:Nystrom-Sampling}
The importance of landmark points in the Nystr\"om method has driven much recent work into various probabilistic and deterministic sampling techniques to improve the accuracy of Nystr\"om-based approximations~\citep{kumar2012sampling,sun2015review}. In this section, we review a few popular sampling methods in the literature. 

The simplest and most common sampling method proposed originally by~\citet{Nystrom2001} was uniform sampling without replacement. In this case, each data point in the data set is sampled with the same probability, i.e., $p_i=\frac{1}{n}$, for $i=1,\ldots,n$. The advantage of this technique is the low computational complexity associated with sampling landmark points. However, it has been shown that uniform sampling does not take into account the nonuniform structure of many data sets. Therefore, sampling mechanisms based on nonuniform  distributions have been proposed to address this problem. Two such examples include: (1) ``Column-norm sampling''~\citep{drineas2006fast}, where $m$ columns of the kernel matrix are sampled with weights proportional to the $\ell_2$ norm of columns
of $\K$ (\emph{not} of the data matrix $\X$),
i.e., $p_i=\|\mathbf{k}_i\|_2^2/\|\K\|_F^2$, and (2) ``diagonal sampling''~\citep{Nystrom_Kernel_Approx}, where the weights are proportional to the corresponding diagonal elements, i.e., $p_i=K_{ii}^2/\sum_{i=1}^{n}K_{ii}^2$. The former requires $\order(n^2)$ time and space to find the nonuniform distribution, while the latter requires $\order(n)$ time and space.
The column-norm sampling method requires computing the entire kernel matrix $\K$, which negates one of the principal benefits of the Nystr\"om method.
The diagonal sampling method reduces to the uniform sampling for shift-invariant kernels, such as the Gaussian kernel function, since $K_{ii}=1$ for all $i=1,\ldots,n$. Recently, \citet{gittens2013revisiting} have studied both empirical and theoretical  aspects of uniform and nonuniform sampling on the accuracy of Nystr\"om-based low-rank approximations.

The ``Clustered Nystr\"om method'' proposed by \citet{zhang2010clusteredNys,zhang2008improved} is a popular non-probabilistic approach that uses out-of-sample extensions to select informative landmark points. The key observation of their work is that the Nystr\"om low-rank approximation error depends on the quantization error of encoding the entire data set with the landmark points. For this reason, the Clustered Nystr\"om method sets the landmark points to be the centroids found from K-means clustering. In machine learning and pattern recognition, K-means clustering~\citep{Bishop} is a well-established technique to partition a data set into clusters 
by trying to minimize the total sum of the squared Euclidean distances of each point to the closest cluster center.

To present the main result of Clustered Nystr\"om method, we first explain K-means clustering briefly. Given $\X=[\x_1,\ldots,\x_n]\in\R^{p\times n}$, an $m$-partition of this data set is a collection $\SSS=\{\SSS_1,\ldots,\SSS_m\}$ of $m$ disjoint and nonempty sets (each representing a cluster) such that their union covers the entire data set. Each cluster can be defined by a
cluster center, which is the sample mean of data points in that cluster. Thus, the goal of K-means clustering is to minimize the following: 
\begin{eqnarray*}
E\left(\X,\SSS\right)=\sum_{i=1}^{n}\|\x_i-\mu(\x_i)\|_2^2,\label{eq:k-means-obj}
\end{eqnarray*}
where $\mu(\x_i)\in\R^{p}$ represents the centroid of the cluster to which the data point $\x_i$ is assigned, and hence depends on $\SSS$. The optimal clustering $\SSS^{opt}$ is the solution of following NP-hard optimization problem~\citep{Bishop}:
\begin{eqnarray}
\SSS^{opt}=\argmin_{\SSS}E\left(\X,\SSS\right).\label{eq:optimization-kmeans}
\end{eqnarray}
In practice, Lloyd's algorithm~\citep{lloyd1982least}, also known as the K-means clustering algorithm, is used to solve the optimization problem in Equation \ref{eq:optimization-kmeans}. The K-means clustering algorithm is an iterative procedure which consists of two steps: (1) data points are assigned to the nearest cluster centers, and (2) the cluster centers are updated based on the most recent assignment of the data points. The objective function decreases at every step, and so the procedure is guaranteed to terminate since there are only finitely many partitions.
Typically, only a few iterations are needed to converge to a locally optimal solution. The quality of clustering can be improved by using well-chosen initialization, such as K-means++ initialization~\citep{kmeans_plusplus}.

Now, we present the result of the Clustered Nystr\"om method which relates the Nystr\"om approximation error (in terms of the Frobenius norm) to the quantization error induced by encoding the data set with landmark points \citep{zhang2010clusteredNys}.
\begin{proposition}[Clustered Nystr\"om Method] \label{thm:clusteredNys}
	Assume that the kernel function $\kappa$ satisfies the following property:
	\begin{eqnarray}
	\left(\kappa(\mathbf{a},\mathbf{b})-\kappa(\mathbf{c},\mathbf{d})\right)^2\leq\eta \left(\|\mathbf{a}-\mathbf{c}\|_2^2+\|\mathbf{b}-\mathbf{d}\|_2^2\right),\;\;\mathbf{a},\mathbf{b},\mathbf{c},\mathbf{d}\in\R^p\label{eq:kernel-cluster}
	\end{eqnarray}
	where $\eta$ is a constant depending on $\kappa$. Consider the data set $\X=[\x_1,\ldots,\x_n]\in\R^{p\times n}$ and the landmark set $\Z=[\z_1,\ldots,\z_m]\in\R^{p\times m}$ which partitions the data set $\X$ into $m$ clusters $\SSS=\{\SSS_1,\ldots,\SSS_m\}$. Let $\mu(\x_i)$ denote the closest landmark point to each data point $\x_i$:
	\begin{eqnarray*}
	\mu\left(\x_i\right)=\argmin_{\z_j\in\{\z_1,\ldots,\z_m\}} \|\x_i-\z_j\|_2.
	\end{eqnarray*}
	Consider the kernel matrix $\K\in\R^{n\times n}$, $K_{ij}=\kappa(\x_i,\x_j)$, and the Nystr\"om approximation $\CC\WW^\dagger\CC^T$, where $C_{ij}=\kappa(\x_i,\z_j)$ and $W_{ij}=\kappa(\z_i,\z_j)$.
	The approximation error in terms of the Frobenius norm is upper bounded:
	\begin{eqnarray}
	\mathcal{E}=\|\K-\CC\WW^\dagger\CC^T\|_F\le \eta_1\sqrt{E\left(\X,\SSS\right)}+\eta_2E\left(\X,\SSS\right)
	\end{eqnarray}
	where $\eta_1$ and $\eta_2$ are two constants and $E(\X,\SSS)$ is the total quantization error of encoding each data point $\x_i$ with the closest landmark point $\mu(\x_i)$:
	\begin{eqnarray}
	E\left(\X,\SSS\right)=\sum_{i=1}^{n} \|\x_i-\mu\left(\x_i\right)\|_2^2. \label{eq:quantization-landmark}
	\end{eqnarray}
\end{proposition}
In~\citep{zhang2010clusteredNys}, it is shown that for a number of widely used kernel functions, e.g., 
linear, polynomial, and Gaussian, the property in Equation \ref{eq:kernel-cluster} is satisfied. Based on Proposition~\ref{thm:clusteredNys}, the Clustered Nystr\"om method tries to minimize the total quantization error in Equation \ref{eq:quantization-landmark}---and thus the Nystr\"om approximation error---by performing the K-means algorithm on the $n$ data points $\x_1,\ldots,\x_n$. The resulting $m$ cluster centers are then chosen as the landmark points to construct matrices $\CC$ and $\WW$ and generate the low-rank approximation $\G=\CC\WW^\dagger\CC^T$. One benefit of the approach is that the full kernel matrix $\K$ is never formed.

\section{Improved Nystr\"om Approximation via QR Decomposition}\label{sec:nys-eig}
In Section~\ref{sec:Nystrom}, we explained the Nystr\"om method to compute rank-$r$ approximations of SPSD kernel matrices based on a set of landmark points. For a data set of size $n$ and a small  set of $m$ landmark points ($m\geq r$), two matrices $\CC\in\R^{n\times m}$ and $\WW\in\R^{m\times m}$ are constructed to form the low-rank approximation of $\K\in\R^{n\times n}$: $\G=\CC\WW^\dagger\CC^T$, where $\text{rank}(\G)\leq m$.

Although the final goal is to find an approximation that has rank no greater than $r$, it is often preferred to select $m>r$ landmark points and then restrict the resultant approximation to have rank at most $r$, e.g.,~\citep{sun2015review,li2015large,wang2013improving}. The main intuition is that selecting $m>r$ landmark points and then restricting the approximation to a lower rank-$r$ space has a regularization effect which can lead to more accurate approximations~\citep{gittens2013revisiting}. For example,  Proposition~\ref{thm:clusteredNys} states that the approximation error $\|\K-\G\|_F$ is a function of the total quantization error induced by encoding data points with the set of $m$ landmark points. Obviously, the more landmark points are selected, the total quantization error becomes smaller and thus the quality of rank-$r$ approximation can be improved. Therefore, it is important to use an efficient and accurate method to restrict the matrix $\G$ to have rank at most $r$. 

In the standard Nystr\"om method presented in Algorithm \ref{alg:StandardNys}, the rank of matrix $\G$ is restricted by computing the best rank-$r$ approximation of the inner matrix $\WW$: $\G_{(r)}^{nys}=\CC\WW_{(r)}^\dagger\CC^T$. Since the inner matrix in the representation of $\G_{(r)}^{nys}$ has rank no greater than $r$, it follows that $\G_{(r)}^{nys}$ has rank at most $r$. The main benefit of this technique is the low computational cost of performing an exact eigenvalue decomposition or SVD on a relatively small matrix of size $m\times m$. However, the standard Nystr\"om method totally ignores the structure of the matrix $\CC$ and is solely based on ``filtering'' $\WW$. In fact, since the rank-$r$ approximation $\G_{(r)}^{nys}$ does not utilize the full knowledge of matrix $\CC$, the selection of more landmark points does not guarantee an improved low-rank approximation in the standard Nystr\"om method.

To solve this problem, we present an efficient method to compute the best rank-$r$ approximation of the matrix $\G=\CC\WW^\dagger\CC^T$, for given matrices $\CC\in\R^{n\times m}$ and $\WW\in\R^{m\times m}$. In contrast with the standard Nystr\"om method, our proposed approach takes advantage of both matrices $\CC$ and $\WW$. To begin, let us consider the best rank-$r$ approximation of the matrix $\G$:
\begin{eqnarray}
\G_{(r)}^{opt}&= &\argmin_{\G':\;\text{rank}(\G')\leq r} \|\CC\WW^\dagger\CC^T-\G'\|_F \nonumber\\
&\overset{(a)}{=} &  \argmin_{\G':\;\text{rank}(\G')\leq r} \|\QQ\underbrace{\RR\WW^\dagger\RR^T}_{m\times m}\QQ^T-\G'\|_F\nonumber\\
&\overset{(b)}{=} & \argmin_{\G':\;\text{rank}(\G')\leq r} \|\left(\QQ\VV'\right)\SIGMA'\left(\QQ\VV'\right)^T-\G'\|_F\nonumber\\
& = & \left(\QQ \VV'_r\right) \SIGMA'_r\left(\QQ \VV'_r\right)^T,\label{eq:optimal-nys}
\end{eqnarray}
where (a) follows from the QR decomposition of $\CC\in\R^{n\times m}$; $\CC=\QQ\RR$, where $\QQ\in\R^{n\times m}$ and $\RR\in\R^{m\times m}$. To get (b), the eigenvalue decomposition of the $m\times m$ matrix $\RR\WW^\dagger\RR^T$ is computed,  $\RR\WW^\dagger\RR^T=\VV'\SIGMA'\VV'^T$, where the diagonal matrix $\SIGMA'\in\R^{m\times m}$ contains $m$ eigenvalues in descending order on the main diagonal and the columns of $\VV'\in\R^{m\times m}$ are the corresponding eigenvectors. Moreover, we note that the columns of  $\QQ\VV'\in\R^{n\times m}$ are orthonormal because both $\QQ$ and $\VV'$ have orthonormal columns:
\begin{eqnarray*}
\big(\QQ\VV'\big)^T\big(\QQ\VV'\big)=\VV'^T\big(\QQ^T\QQ\big)\VV'=\VV'^T\VV'=\eye_{m\times m}.
\end{eqnarray*}
Thus, the decomposition $(\QQ\VV')\SIGMA'(\QQ\VV')^T$ contains the $m$ eigenvalues and orthonormal eigenvectors of the Nystr\"om approximation $\CC\WW^\dagger\CC^T$. Based on the Eckart-Young theorem, the best rank-$r$ approximation of $\G=\CC\WW^\dagger\CC^T$ is then computed using the $r$ leading eigenvalues $\SIGMA'_r\in\R^{r\times r}$ and corresponding eigenvectors $\QQ\VV'_r\in\R^{n\times r}$, as given in Equation \ref{eq:optimal-nys}. Thus, the estimates of the top $r$ eigenvalues and eigenvectors of the kernel matrix $\K$ from the Nystr\"om approximation $\CC\WW^\dagger\CC^T$ are obtained as follows:
\begin{eqnarray}
\widehat{\UU}_r^{opt}=\QQ\VV'_r,\;\;\widehat{\Lam}_r^{opt}=\SIGMA'_r.
\end{eqnarray}
These estimates can also be used to approximate the kernel matrix as $\K \approx\LL^{opt}\left(\LL^{opt}\right)^T$, where $\LL^{opt}=\widehat{\UU}_r^{opt}\big(\widehat{\Lam}_r^{opt}\big)^{1/2}$.  

\begin{algorithm}[t]
	\caption{Nystr\"om via QR Decomposition}
	\label{alg:NysQR}
	\textbf{Input:} data set $\X$, landmark points $\Z$, kernel function $\kappa$, target rank $r$
	
	\textbf{Output:} estimates of $r$ leading eigenvectors and eigenvalues of the kernel matrix $\K\in\R^{n\times n}$: $\widehat{\UU}_r^{opt}\in\R^{n\times r}$, $\widehat{\Lam}_r^{opt}\in\R^{r\times r}$
	\begin{algorithmic}[1]
		\STATE  Form two matrices $\CC$ and $\WW$: $C_{ij}=\kappa(\x_i,\z_j)$, $W_{ij}=\kappa(\z_i,\z_j)$
		\STATE Perform the QR decomposition: $\CC=\QQ\RR$
		\STATE Compute the eigenvalue decomposition: $\RR\WW^\dagger\RR^T=\VV'\SIGMA'\VV'^T$
		\STATE $\widehat{\UU}_r^{opt}=\QQ\VV'_r$ and $\widehat{\Lam}_r^{opt}=\SIGMA'_r$
	\end{algorithmic}
\end{algorithm}

The overall procedure to estimate the $r$ leading eigenvalues/eigenvectors of the kernel matrix $\K$  based on a set of landmark points $\Z\in\R^{p\times m}$, $m\geq r$, is presented in Algorithm \ref{alg:NysQR}. The time complexity of this method is $\order(pnm+nm^2+m^3+nmr)$, where $\order(pnm)$ represents the cost to form matrices $\CC$ and $\WW$.
The complexity of the QR decomposition is $\order(nm^2)$ and it takes $\order(m^3)$ time to compute the eigenvalue decomposition of $\RR\WW^\dagger\RR^T$. Finally, the cost to compute the matrix multiplication $\QQ\VV'_r$ is $\order(nmr)$.

We can compare the computational complexity of our proposed Nystr\"om method via QR decomposition (Algorithm \ref{alg:NysQR}) with that of the standard Nystr\"om method (Algorithm \ref{alg:StandardNys}). Since our focus in this paper is on large-scale data sets with $n$ large, we only consider terms involving $n$ which lead to dominant computation costs. Based on the discussion in Section~\ref{sec:Nystrom}, it takes $\mathcal{C}_{nys}=\order(pnm+nmr+nr^2)$ time to compute the eigenvalue decomposition using the standard Nystr\"om method. For our proposed method, the cost of eigenvalue decomposition is $\mathcal{C}_{opt}=\order(pnm+nmr+nm^2)$. Thus, for data of even moderate dimension with $p\gtrsim m$, the dominant term in both $\mathcal{C}_{nys}$ and $\mathcal{C}_{opt}$ is $\order(pnm)$. This means that the increase in computation cost of our method ($nm^2$ vs.~$nr^2$) becomes less significant when the number of landmark points $m$ is close to the target rank $r$. 

In the rest of this section, we compare the performance and efficiency of our proposed method presented in Algorithm \ref{alg:NysQR} with Algorithm \ref{alg:StandardNys} on three examples. As we will see, our proposed method yields more accurate decompositions than the standard Nystr\"om method for small values of $m$, such as $m=2r$.

\subsection{Toy Example}
It is \emph{always} true that for any kernel matrix $\K$, 
$\|\K-\G_{(r)}^{opt}\|_F \le \|\K-\G_{(r)}^{nys}\|_F$ (this is also true in the spectral norm), due to the best-approximation properties of our estimator.
We can show, using examples, that this inequality can be quite large.

In the first example, we consider a small kernel matrix of size $3\times 3$: 
\begin{eqnarray*}
\K=\left[\begin{array}{ccc}
1 & 0 & 10\\
0 & 1.01 & 0\\
10 & 0 & 100
\end{array}\right].
\end{eqnarray*}
Such a matrix could arise, for example, using the polynomial kernel with parameters $c=0$ and $d=1$ and the data matrix:
\begin{eqnarray*}
\X=\frac{1}{\sqrt{2}}\left[\begin{array}{ccc}
1 & 0 & 10\\
0 & \sqrt{2\cdot 1.01} & 0\\
1 & 0 & 10
\end{array}\right].
\end{eqnarray*}
Here, the goal is to compute the rank $r=1$ approximation of $\K$. Suppose that  $m=2$ columns of the kernel matrix are sampled uniformly, e.g., the first and second columns. Then, we have:
\begin{eqnarray*}
\CC=\left[\begin{array}{cc}
1 & 0\\
0 & 1.01\\
10 & 0
\end{array}\right],\;\;\WW=\left[\begin{array}{cc}
1 & 0\\
0 & 1.01
\end{array}\right].
\end{eqnarray*}
In the standard Nystr\"om method, the best rank-$1$ approximation of the inner matrix $\WW$ is first computed\footnote{One might ask if it is better to \emph{first} find $\WW^\dagger$ and \emph{then} find the best rank-$r$ approximation of $\WW^\dagger$. This generally does not help, and one can construct similar toy examples where this approach does arbitrarily poorly as well.}. Then, based on Equation \ref{eq:Nys-rank-r}, the rank-$1$ approximation of the kernel matrix in the standard Nystr\"om method is given by:
\begin{eqnarray*}
\G_{(1)}^{nys}= \left[\begin{array}{cc}
1 & 0\\
0 & 1.01\\
10 & 0
\end{array}\right]\left[\begin{array}{cc}
0 & 0\\
0 & \frac{1}{1.01}
\end{array}\right]
\left[\begin{array}{ccc}
1 & 0 & 10\\
0 & 1.01 & 0
\end{array}\right]
=
\left[\begin{array}{ccc}
0 & 0 & 0\\
0 & 1.01 & 0\\
0 & 0 & 0
\end{array}\right].
\end{eqnarray*}

The normalized kernel approximation error in terms of the Frobenius norm is large:  $\|\K-\G_{(1)}^{nys}\|_F/\|\K\|_F=0.99$.
On the other hand, using the same matrices $\CC$ and $\WW$, our proposed method first computes the QR decomposition of $\CC=\QQ\RR$:
\begin{eqnarray*}
\QQ=\left[\begin{array}{cc}
\frac{1}{\sqrt{101}} & 0\\
0 & 1\\
\frac{10}{\sqrt{101}} & 0
\end{array}\right],\;\;\RR=\left[\begin{array}{cc}
\sqrt{101} & 0\\
0 & 1.01
\end{array}\right].
\end{eqnarray*}
Then, the product of three matrices $\RR\WW^\dagger\RR^T$ is computed to find its eigenvalue decomposition $\RR\WW^\dagger\RR^T=\VV'\SIGMA'\VV'^T$:
\begin{eqnarray*}
\RR\WW^\dagger\RR^T&= & \left[\begin{array}{cc}
\sqrt{101} & 0\\
0 & 1.01
\end{array}\right]
\left[\begin{array}{cc}
1 & 0\\
0 & \frac{1}{1.01}
\end{array}\right]
\left[\begin{array}{cc}
\sqrt{101} & 0\\
0 & 1.01
\end{array}\right]
\nonumber
\\
&= &\left[\begin{array}{cc}
101 & 0\\
0 & 1.01
\end{array}\right]\nonumber \\
&= & \underbrace{\left[\begin{array}{cc}
	1 & 0\\
	0 & 1
	\end{array}\right]}_{\VV'}\underbrace{\left[\begin{array}{cc}
	101 & 0\\
	0 & 1.01
	\end{array}\right]}_{\SIGMA'}\underbrace{\left[\begin{array}{cc}
	1 & 0\\
	0 & 1
	\end{array}\right]}_{\VV'^T}.
\end{eqnarray*}
Finally, the rank-$1$ approximation of the kernel matrix in our proposed method is obtained by using Equation \ref{eq:optimal-nys}:
\begin{eqnarray*}
\G_{(1)}^{opt}=\left[\begin{array}{cc}
\frac{1}{\sqrt{101}} & 0\\
0 & 1\\
\frac{10}{\sqrt{101}} & 0
\end{array}\right]\left[\begin{array}{cc}
101 & 0\\
0 & 0
\end{array}\right]
\left[\begin{array}{ccc}
\frac{1}{\sqrt{101}} & 0 & \frac{10}{\sqrt{101}}\\
0 & 1 & 0
\end{array}\right]=\left[\begin{array}{ccc}
1 & 0 & 10\\
0 & 0 & 0\\
10 & 0 & 100
\end{array}\right],
\end{eqnarray*}
where $\|\K-\G_{(1)}^{opt}\|_F/\|\K\|_F=0.01$. 
In fact, one can show that our approximation is the same as the best rank-$1$ approximation formed using full knowledge of $\K$, i.e., $\G_{(1)}^{opt}=\K_{(1)}$. 
Furthermore, clearly we can tweak this toy example to make the error $\|\K-\G_{(1)}^{opt}\|_F/\|\K\|_F=\epsilon$ and $\|\K-\G_{(1)}^{nys}\|_F/\|\K\|_F=1-\epsilon$ for any $\epsilon>0$.
This example demonstrates that
``Nystr\"om via QR Decomposition'' produces much more accurate rank-$1$ approximation of the kernel matrix with same matrices $\CC$ and $\WW$ used in the standard Nystr\"om method. 

\subsection{Synthetic Data Set}
As shown in Figure \ref{fig:syn-data-R2}, we consider a synthetic data set consisting of $n=4000$ data points in $\R^2$ that are nonlinearly separable. Therefore, a nonlinear kernel function is employed to find an embedding of these points so that linear learning algorithms can be applied to the mapped data points. To do this, we use the polynomial kernel function with the degree $d=2$ and the constant $c=0$, i.e., $\kappa(\x_i,\x_j)=\langle\x_i,\x_j\rangle^2$. Next, a low-rank approximation of the kernel matrix in the form of $\K\approx\LL\LL^T$, $\LL\in\R^{n\times r}$, is computed by using the Nystr\"om method. The $n$ rows of $\LL$ represent the virtual samples or mapped data points~\citep{zhang2012scaling,golts2016linearized,pourkamali2016randomized}. Given a suitable kernel function and accurate low-rank approximation technique, the $n$ rows of $\LL$ in $\R^r$ are linearly separable. In this example, we set the target rank $r=2$ so that we can easily visualize the resultant mappings.

We measure the approximation accuracy by using the normalized kernel approximation error defined as $\|\K-\LL\LL^T\|_F/\|\K\|_F$, where the matrix $\LL$ is obtained by using the standard Nystr\"om method and our proposed method ``Nystr\"om via QR Decomposition''.
In Figure \ref{fig:syn-error}, the mean and standard deviation of the normalized kernel approximation error over $50$ trials for varying number of landmark points $m$ are reported. In each trial, the $m$ landmark points are chosen uniformly at random without replacement from the input data. Both our method and the standard Nystr\"om method share same matrices $\CC$ and $\WW$ for a fair comparison.  As we expect, the accuracy of our Nystr\"om via QR decomposition is exactly the same as the standard Nystr\"om method for $m=r=2$. As the number of landmark points $m$ increases, the accuracy of standard Nystr\"om method improves and it slowly gets closer to the accuracy of exact eigenvalue decomposition or SVD. However, our proposed method reaches the accuracy of SVD even for $m=2r=4$. In fact, we observe that the approximation error of our method by using $m=4$ landmark points is better than the accuracy of  standard Nystr\"om method with $m=40$. For this example, our proposed rank-$r$ approximation technique in Algorithm \ref{alg:NysQR} is more accurate and memory efficient than the standard Nystr\"om method with at least one order of magnitude savings in memory.
\begin{figure}[t]
	\begin{centering}
		\subfloat[Original data]{\begin{centering}
				\includegraphics[scale=0.40]{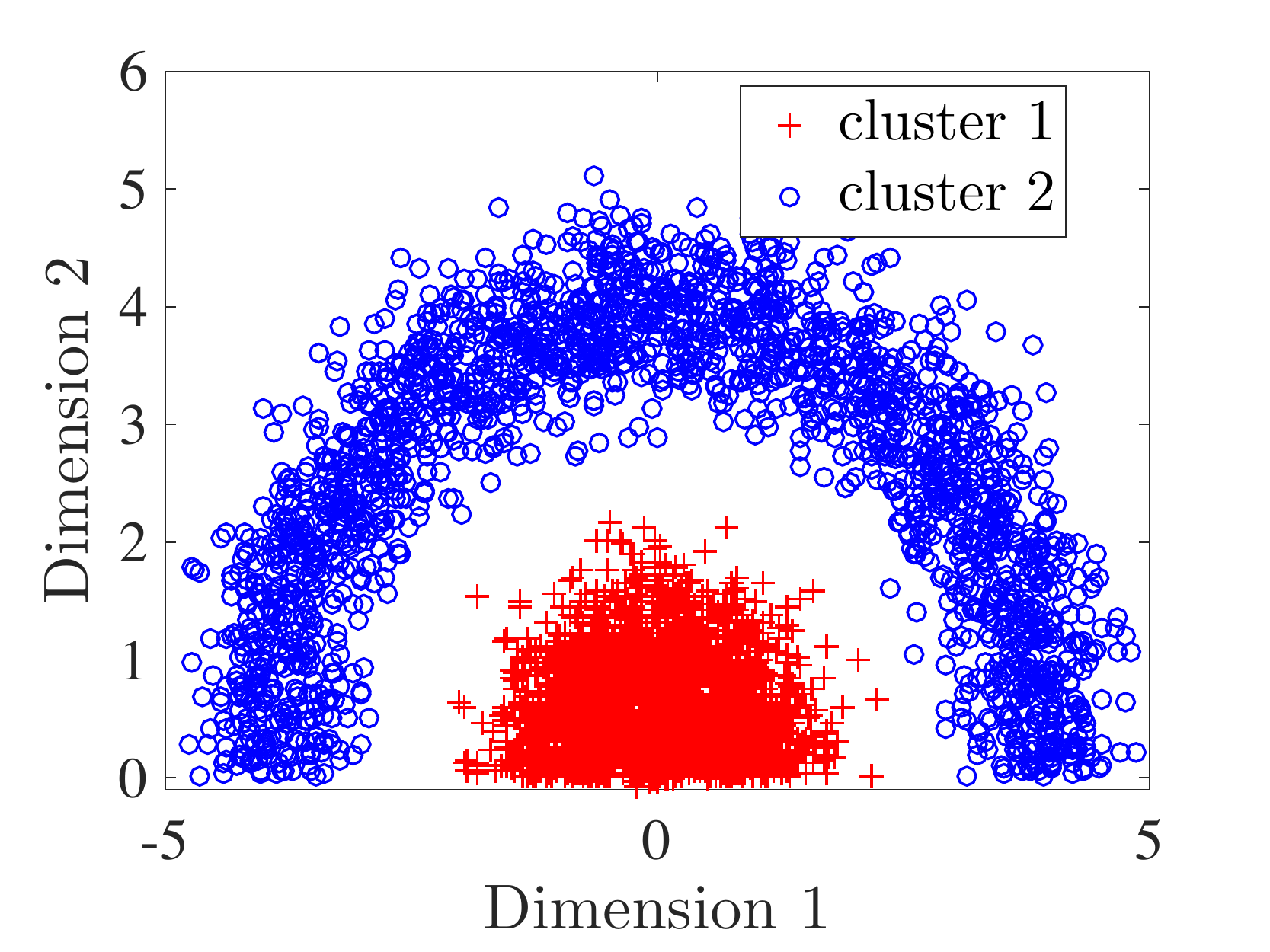}
				\par\end{centering}\label{fig:syn-data-R2}
		}\subfloat[Kernel approximation error]{\begin{centering}
			\includegraphics[scale=0.40]{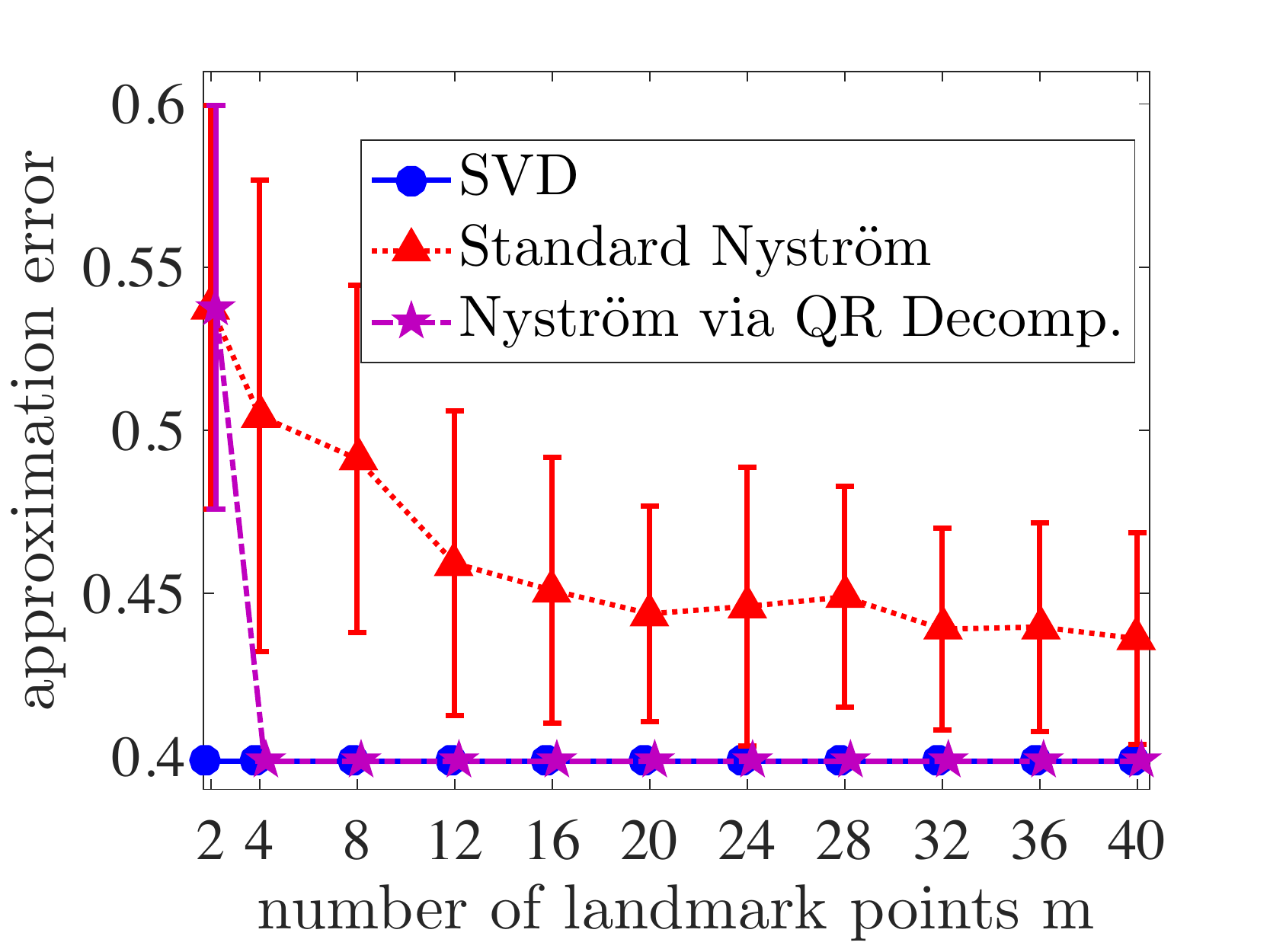}\label{fig:syn-error}
			\par\end{centering}
	}
	\par\end{centering}

\begin{centering}
	\subfloat[Standard Nystr\"om]{\begin{centering}
			\includegraphics[scale=0.40]{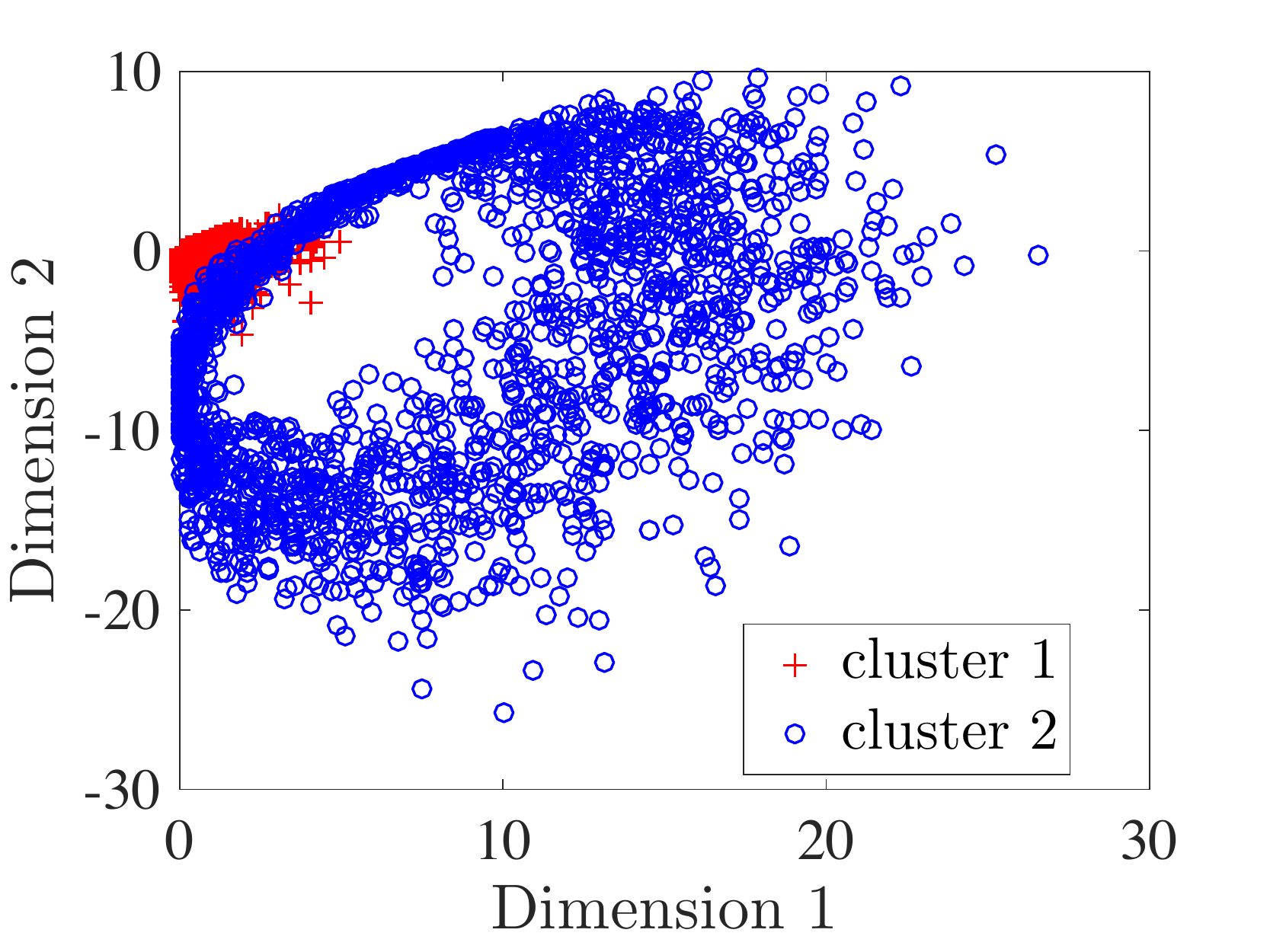}\label{fig:syn-data-sta}
			\par\end{centering}
	}\subfloat[Nystr\"om via QR Decomposition]{\begin{centering}
		\includegraphics[scale=0.40]{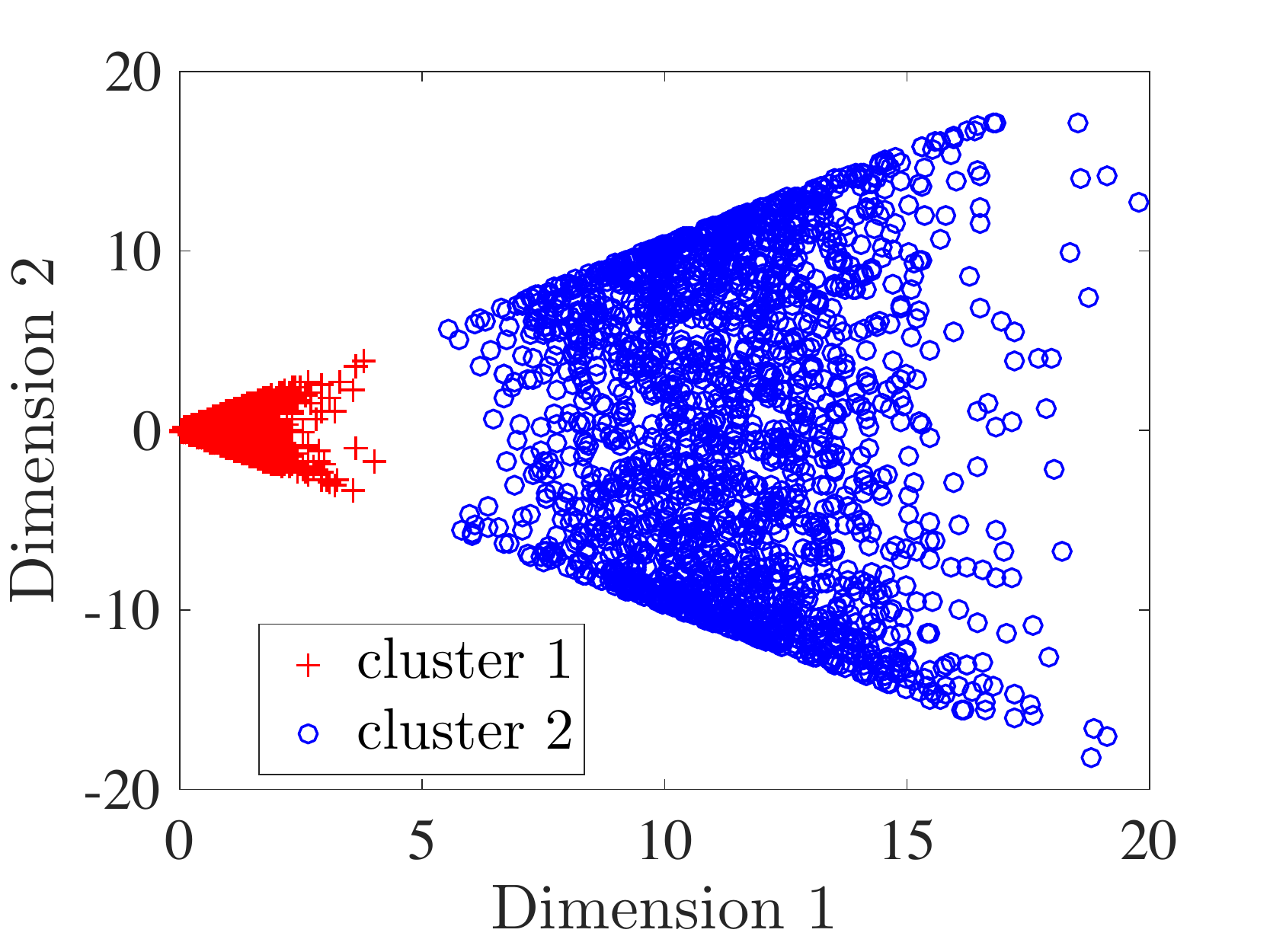}\label{fig:syn-data-qr}
		\par\end{centering}
}
\par\end{centering}

\caption{The standard Nystr\"om method is compared with our proposed ``Nystr\"om via QR Decomposition'' on the synthetic data set ($p=2$, $n=4000$). The polynomial kernel function $\kappa(\x_i,\x_j)=\langle\x_i,\x_j\rangle^2$ is used to find nonlinear mappings of the original data points by using the rank-$2$ approximation of the kernel matrix $\K\approx\LL\LL^T$, $\LL\in\R^{n\times 2}$. The bottom row uses $m=4$ landmark points.}

\end{figure}

Finally, we visualize the mapped data points using both methods for fixed $m=4$. In Figure \ref{fig:syn-data-sta} and Figure \ref{fig:syn-data-qr}, the rows of $\LL^{nys}\in\R^{n\times 2}$ and $\LL^{opt}\in\R^{n\times 2}$  are plotted, respectively. The rows of $\LL^{opt}\in\R^{n\times 2}$ in the ``Nystr\"om via QR Decomposition'' method are linearly separable which is desirable for kernel-based learning. But, the rows of $\LL^{nys}\in\R^{n\times 2}$ are not linearly separable due to the poor performance of the standard Nystr\"om method.

\subsection{Real Data Set: \dataset{satimage}}
In the last example, we use the \dataset{satimage} data set~\citep{CC01a} with $p=36$ and $n=4435$. We duplicate each data point four times to increase to $n=17,\!740$ in order to have a more meaningful comparison of computation times.
The kernel matrix is formed using the Gaussian kernel function $\kappa\left(\x_i,\x_j\right)=\exp\left(-\|\x_i-\x_j\|_2^2/c\right)$  where the parameter $c$ is chosen as the averaged squared distance between all the data points and the sample mean~\citep{zhang2010clusteredNys}. The $m$ landmark points are chosen by performing K-means on the original data, following the Clustered Nystr\"om method.

In Figure \ref{fig:QR_error_r2} and Figure \ref{fig:QR_error_r5}, the mean and standard deviation of normalized kernel approximation error are reported over $50$ trials for varying number of landmark points $m$ and two values of the target rank $r=2$ and $r=5$, respectively. 
As expected, 
when the number of landmark points is set to be the same as the target rank, the standard Nystr\"om method and our proposed method have exactly the same approximation error. Interestingly, it is seen that when the number of landmark points $m$ increases, the approximation error does not necessarily decrease in the standard Nystr\"om method as shown in Figure \ref{fig:QR_error_r2}. This is a major drawback of the standard Nystr\"om method because the increase in memory and computation costs imposed by larger $m$ may lead to worse performance.
In contrast, our proposed ``Nystr\"om via QR Decomposition'' outperforms the standard Nystr\"om method for both values of the target rank $r=2$ and $r=5$,
and we know theoretically that performance can only improve as $m$ increases.
Moreover, we see that the accuracy of our method reaches the accuracy of the best rank-$r$ approximation obtained by using the SVD for as few as $m=2r$ landmark points.
\begin{figure}[t]
	\begin{centering}
		\subfloat[Kernel approximation error, $r=2$]{\begin{centering}
				\includegraphics[scale=0.40]{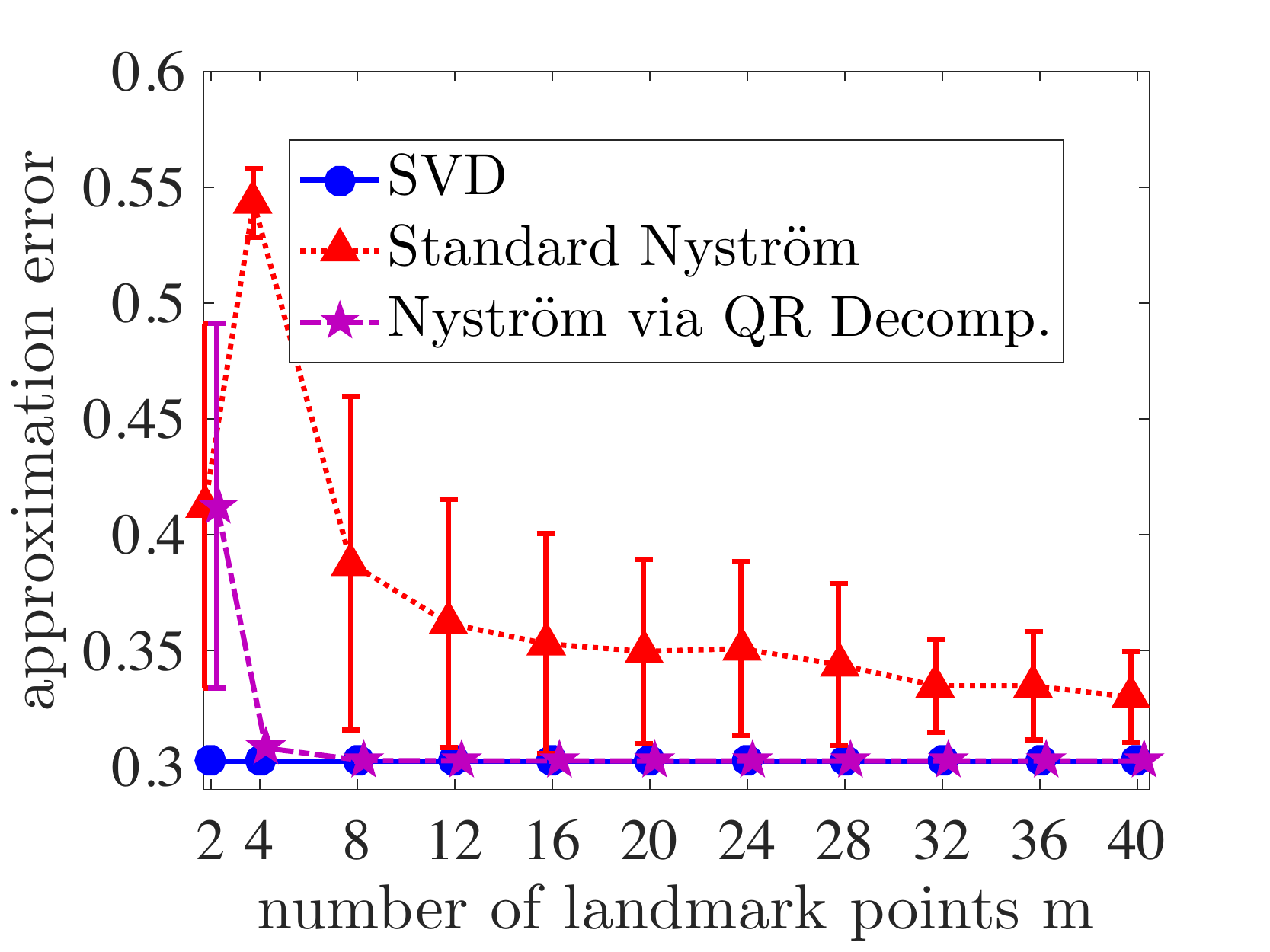}\label{fig:QR_error_r2}
				\par\end{centering}
		}\subfloat[Runtime, $r=2$]{\begin{centering}
			\includegraphics[scale=0.40]{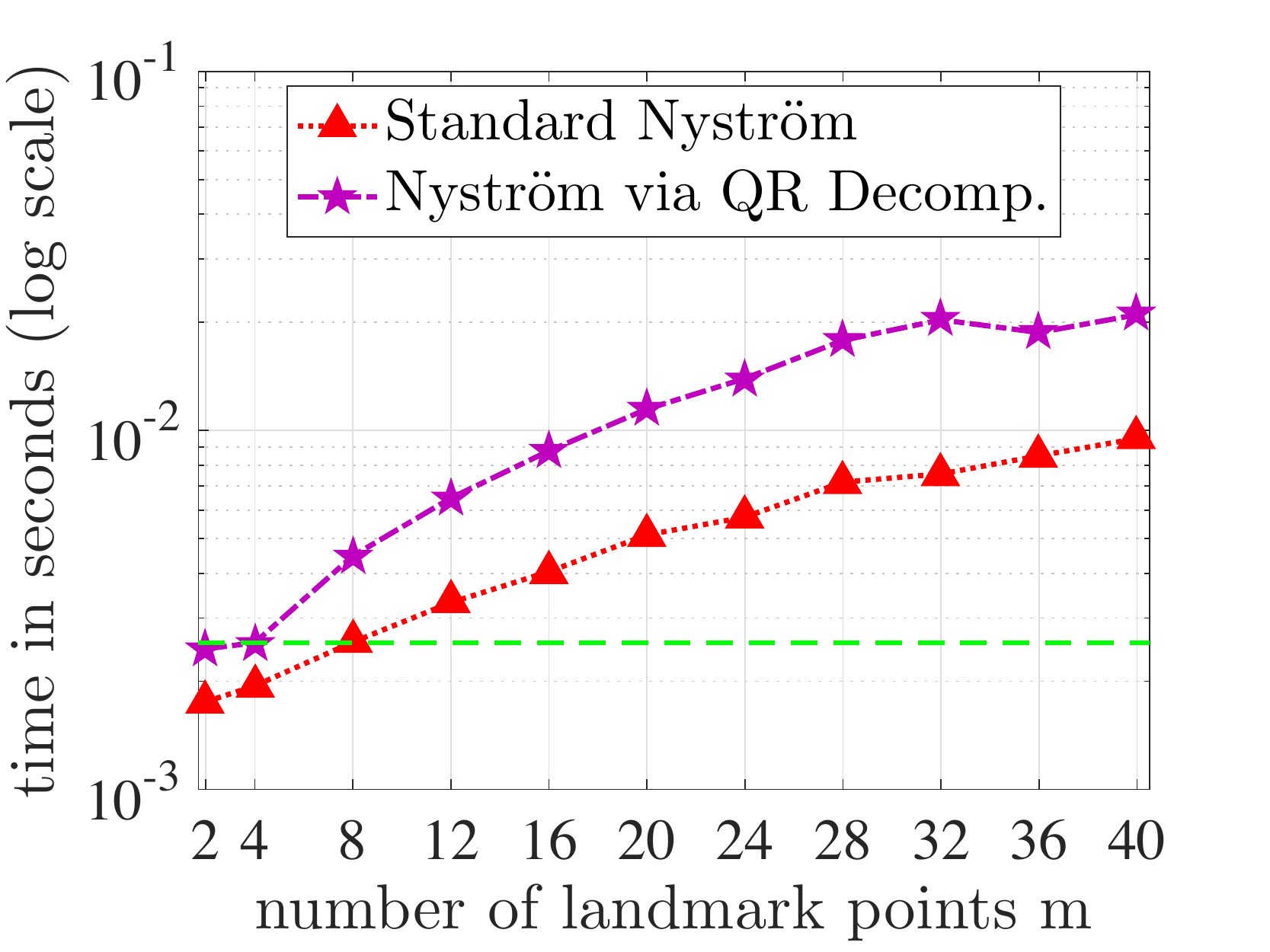}\label{fig:QR_time_r2}
			\par\end{centering}
	}
	\par\end{centering}

\begin{centering}
	\subfloat[Kernel approximation error, $r=5$]{\begin{centering}
			\includegraphics[scale=0.40]{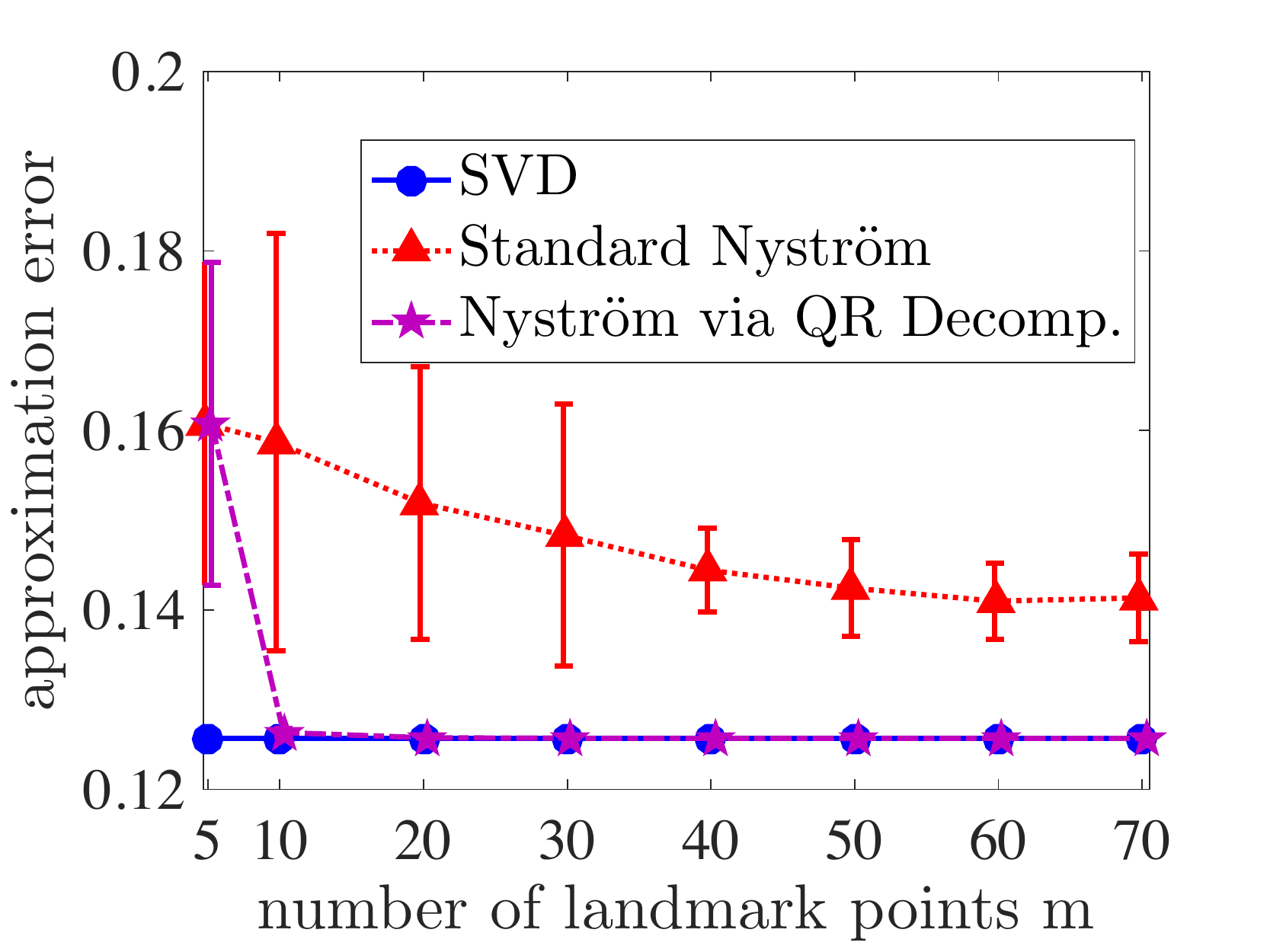}\label{fig:QR_error_r5}
			\par\end{centering}
	}\subfloat[Runtime, $r=5$]{\begin{centering}
		\includegraphics[scale=0.40]{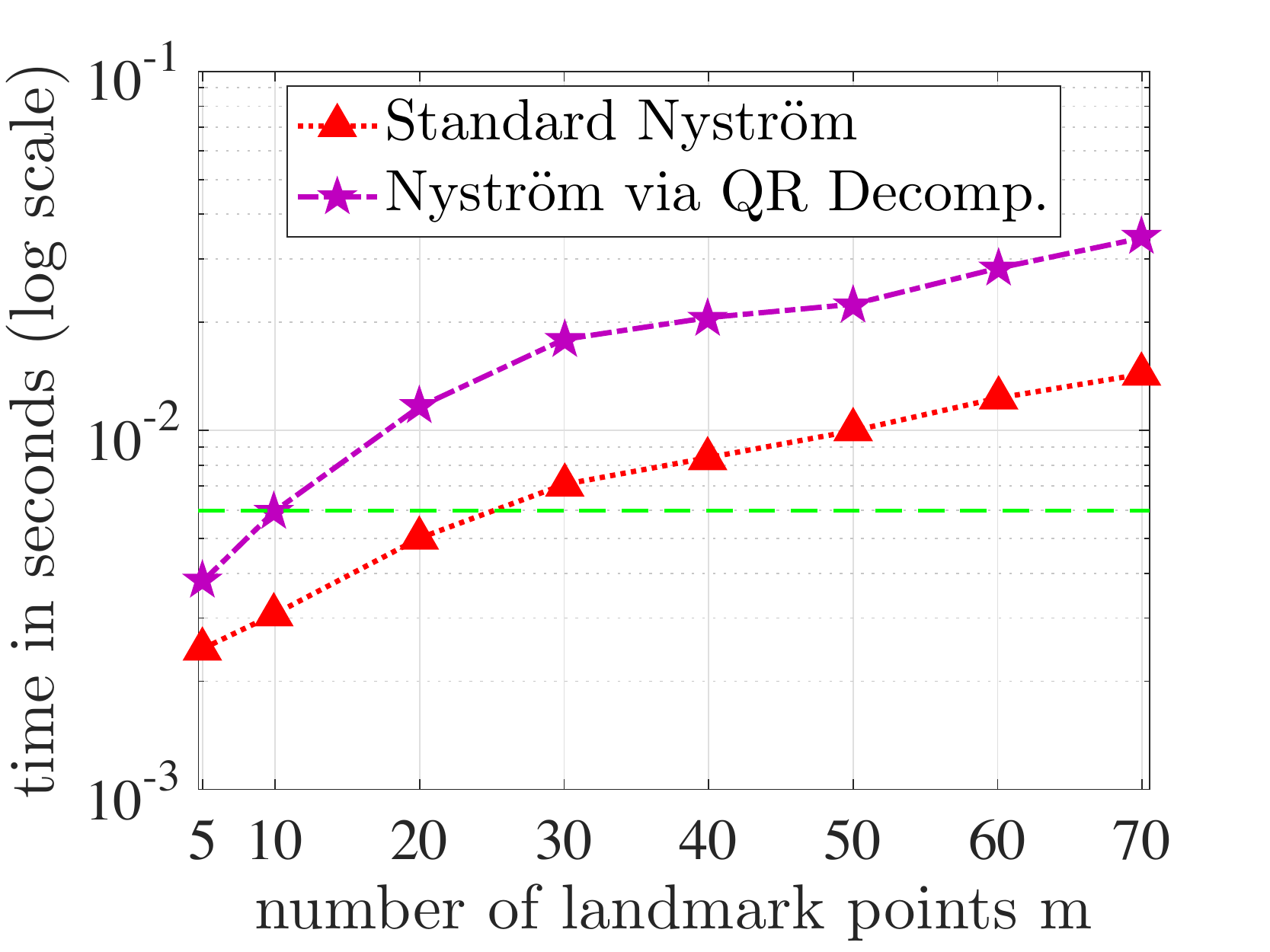}\label{fig:QR_time_r5}
		\par\end{centering}
}
\par\end{centering}

\caption{The standard Nystr\"om method is compared with our proposed ``Nystr\"om via QR Decomposition'' on the satimage data set. Our method yields more accurate low-rank approximations with noticeable memory and computation savings.}
\end{figure}

The runtime of both methods are also compared in Figure \ref{fig:QR_time_r2} and Figure \ref{fig:QR_time_r5} for two cases of $r=2$ and $r=5$, respectively. The reported values are averaged over $50$ trials and they represent the computation cost associated with Algorithm \ref{alg:StandardNys} and Algorithm \ref{alg:NysQR}.
As we explained earlier in this section, the computational complexity of our method $\mathcal{C}_{opt}$ will be slightly increased compared to the standard Nystr\"om method $\mathcal{C}_{nys}$ and this is consistent with the timing results in Figure \ref{fig:QR_time_r2} and Figure \ref{fig:QR_time_r5}. Moreover, we see that the runtime of our method is increased by almost a factor of $2$ even for large values of $m$. 
To have a fair comparison, we draw a dashed green line that determines the values of $m$ for which both methods have the same running time. In Figure \ref{fig:QR_time_r2}, the runtime for $m=4$ in our method is the same as $m=8$ in  the standard Nystr\"om, while our method is much more accurate. Similarly, in Figure \ref{fig:QR_time_r5}, the runtime for $m=10$ in our method is almost the same as $m=30$ in the standard Nystr\"om method. However, our method results in more accurate low-rank approximation of the kernel matrix. 
This dataset further supports that 
our ``Nystr\"om via QR Decomposition'' results in more accurate low-rank approximations than the standard Nystr\"om method with significant memory and computation savings. 


\section{Randomized Clustered Nystr\"om Method}\label{sec:random-clustered-nys}
The selection of informative landmark points is an essential component to obtain accurate low-rank approximations of SPSD matrices in the Nystr\"om method. The Clustered Nystr\"om method~\citep{zhang2010clusteredNys} has been shown to be a powerful technique for generating  highly accurate low-rank approximations compared to uniform sampling and other sampling methods~\citep{kumar2012sampling,sun2015review,iosifidis2016nystrom}. However, the main drawbacks of this method are high memory and computational complexities associated with performing K-means clustering on large-scale data sets. In this section, we introduce an efficient randomized method for generating a set of representative landmark points based on low-dimensional random projections of the original data.
Specifically, our proposed method provides a ``tunable tradeoff'' between the accuracy of Nystr\"om low-rank approximations and the efficiency in terms of memory and computation savings. 

To introduce our proposed method, we begin by explaining the process of generating landmark points in the Clustered Nystr\"om method. As mentioned in Section~\ref{sec:Nystrom-Sampling}, the central idea behind Clustered Nystr\"om is that the approximation error depends on the total quantization error of encoding each data point in the data set with the closest landmark point. Thus, landmark points are chosen to be centroids resulting from the K-means clustering algorithm which partitions the data set into $m$ clusters. Given an initial set of $m$ centroids $\{\muu_j\}_{j=1}^m\in\R^p$, the K-means clustering algorithm iteratively updates assignments and cluster centroids as follows~\citep{Bishop}:
\begin{enumerate}
	\item Update assignments: for $i=1,\ldots,n$
	\begin{eqnarray*}
	\x_i\in\SSS_j \Leftrightarrow j\in\argmin_{j'\in\{1,\ldots,m\}}\|\x_i-\muu_{j'}\|_2
	\end{eqnarray*}
	\item Update cluster centroids: for $j=1,\ldots,m$
	\begin{eqnarray} \label{eq:center}
	\muu_j=\frac{1}{|\SSS_j|}\sum_{\x_i\in\SSS_j}\x_i
	\end{eqnarray}
\end{enumerate}
where $|\SSS_j|$ denotes the number of data points in the cluster $\SSS_j$ and $\muu_j$ is the sample mean of the $j$-th cluster.

For large-scale data sets with large $p$ and/or $n$,
the memory requirements and computation cost of performing the K-means clustering algorithm become expensive~\citep{ailon2009streaming,shindler2011fast,K-means-Feldman}. First, the K-means algorithm requires several passes on the entire data set and thus the data set should often be stored in a centralized location which takes $\order(pn)$ memory. Second, the time complexity of K-means clustering is $\order(pnm)$ per iteration to partition the set of $n$ data points into $m$ clusters~\citep{SketchAndValidateKMeans}. Hence, the high dimensionality of massive data sets provides considerable challenge to the design of memory and computation efficient
alternatives for the Clustered Nystr\"om method.

One promising strategy to address these obstacles is to use random projections of the data for constructing a small set of new features~\citep{DataBaseFriendlyRandomProjection,Pourkamali_ICML_2014,FastCSTracking_PAMI,Pourkamali_DL_SampTA}. In this case, for some parameter $p'<p$, the data matrix $\X$ is multiplied on the left by a random zero-mean matrix $\HH\in\R^{p'\times p}$ 
in order 
to compute a low-dimensional representation:
\begin{eqnarray*}
\widehat{\X}=\HH\X=\left[\HH\x_1,\ldots,\HH\x_n\right]\in\R^{p'\times n}.
\end{eqnarray*}
The columns of $\widehat{\X}=[\widehat{\x}_1,\ldots,\widehat{\x}_n]$ are known as sketches or compressive measurements~\citep{davenport2010signal} and the random map $\HH$ preserves the geometry of data under certain conditions~\citep{ImprovedAnalysis}. The task of clustering is then performed on these low-dimensional data points by minimizing $E(\widehat{\X},\SSS)=\sum_{i=1}^{n}\|\widehat{\x}_i-\mu(\widehat{\x}_i)\|_2^2$, which partitions the data points in the reduced space into $m$ clusters. 
After finding the partition in the reduced space, the same partition is used on the original data points and the cluster centroids in the original space are calculated using Equation \ref{eq:center} at computational cost $\order(np)$. 

In this paper, we introduce a random-projection-type Clustered Nystr\"om method, called ``Randomized Clustered Nystr\"om,'' for generating landmark points. In the first step of our method, a random sign matrix $\HH\in\R^{p'\times p}$ whose entries are independent realizations of $\{\pm 1/\sqrt{p'}\}$ Bernoulli random variables is constructed:
\begin{eqnarray}
H_{ij}=\begin{cases}
+1/\sqrt{p'} & \text{with probability }1/2,\\
-1/\sqrt{p'} & \text{with probability }1/2.
\end{cases}\label{eq:random-matrix}
\end{eqnarray}
Next, the product $\HH\X$ is computed to find the low-dimensional sketches $\widehat{\x}_1,\ldots,\widehat{\x}_n\in\R^{p'}$. The standard implementation of matrix multiplication costs $\order(p'pn)$.  The matrix multiplication can also be performed in parallel which leads to noticeable accelerations in practice~\citep{Martinson_SVD}.  
Moreover, it is possible to use the mailman algorithm~\citep{liberty2009mailman} 
which takes advantage of the binary-nature of $\HH$  
to further speed up the matrix multiplication. In our experiments, we use  Intel MKL BLAS version 11.2.3 
which is bundled with MATLAB, which we found to be sufficiently optimized and does not form a bottleneck in the computational cost.

In the second step, the K-means clustering algorithm is performed on the projected low-dimensional data $\widehat{\X}=[\widehat{\x}_1,\ldots,\widehat{\x}_n]$ to partition the data set:
\begin{eqnarray*}
\widehat{\SSS}^{opt}\approx \argmin_{\SSS} E(\widehat{\X},\SSS), 
\end{eqnarray*}
where $\widehat{\SSS}^{opt}=\{\widehat{\SSS}_1^{opt},\ldots,\widehat{\SSS}_m^{opt}\}$ is the resulting $m$-partition. 
We cannot guarantee that K-means returns the globally optimal partition as the problem is NP-hard~\citep{DasguptaKmeans} but seeding using K-means++~\citep{kmeans_plusplus} guarantees a partition with expected objective within a $\log(m)$ factor of the optimal one,
and other variants of K-means, under mild assumptions~\citep{ostrovsky2012effectiveness}, can
either efficiently guarantee a solution within a constant factor of optimal, or 
guarantee solutions arbitrarily close to optimal, so-called polynomial-time approximation schemes (PTAS).
Lastly, the landmark points are generated by computing the sample mean of data points:
\begin{eqnarray}
\z_j=\big(1/|\widehat{\SSS}_j^{opt}|\big)\sum_{\x_i\in\widehat{\SSS}_j^{opt}}\x_i,\;\;j=1,\ldots,m.\label{eq:z-reduced-space}
\end{eqnarray}

The proposed ``Randomized Clustered Nystr\"om'' method is summarized in Algorithm \ref{alg:RandomizedClustered}. In our method, the ``compression factor'' $\gamma$ is defined as the ratio of parameter $p'$ to the ambient dimension $p$, i.e., $\gamma:=p'/p<1$.
Regarding the memory complexity, our method requires only two passes on the data set $\X$, the first to compute the low-dimensional sketches (step~\ref{alg:pass1}), and the second for the  sample mean (step~\ref{alg:pass2}). In fact, our Randomized Clustered Nystr\"om only stores the low-dimensional sketches which takes $\order(p'n)$ space, whereas the Clustered Nystr\"om method has memory complexity of $\order(pn)$,
meaning our method reduces the memory complexity by a factor of $1/\gamma$.
In terms of time complexity, the computation cost of K-means on the dimension-reduced data in our method is $\order(p'nm)$ per iteration compared to the cost $\order(pnm)$ in the Clustered Nystr\"om method, so the speedup is up to $1/\gamma$ (the exact amount depends on the number of iterations, since we must amortize the cost of the one-time  matrix multiply $\HH\X$).

\begin{algorithm}[t]
	\caption{Randomized Clustered Nystr\"om}
	\label{alg:RandomizedClustered}
	\textbf{Input:} data set $\X$, number of landmark points $m$, compression factor $\gamma<1$
	
	\textbf{Output:} landmark points $\Z$
	\begin{algorithmic}[1]
		\STATE Set $p'=\gamma p$ (round to nearest integer)
		\STATE  Generate a random sign matrix $\HH\in\R^{p'\times p}$ as in Equation \ref{eq:random-matrix}
		\STATE \label{alg:pass1}Compute $\widehat{\X}=\HH\X\in\R^{p'\times n}$
		\STATE \label{alg:k-means-low}Perform K-means clustering on $\widehat{\X}=[\widehat{\x}_1,\ldots,\widehat{\x}_n]$ to get $\widehat{\SSS}^{opt}$
		\STATE \label{alg:pass2}Compute the sample mean  in the original space as in Equation \ref{eq:z-reduced-space}
		\STATE $\Z=[\z_1,\ldots,\z_m]\in\R^{p\times m}$
	\end{algorithmic}
\end{algorithm}

Thus, our proposed method for generating landmark points provides a tunable parameter $\gamma$ to reduce the memory and computation cost of the Clustered Nystr\"om method. Next, we study and characterize the ``tradeoffs'' between accuracy of low-rank approximations and the memory/computation savings in our proposed method. In particular, the following theorem presents an error bound on the Nystr\"om low-rank approximation for a set of landmark points generated via our Randomized Clustered Nystr\"om method (Algorithm ~\ref{alg:RandomizedClustered}).

\begin{theorem}[Randomized Clustered Nystr\"om Method]
	\label{thm:randomized-clustered-nys}
	Assume that the kernel function $\kappa$ satisfies 
	Equation \ref{eq:kernel-cluster}.
	Consider the data set $\X=[\x_1,\ldots,\x_n]\in\R^{p\times n}$ and the kernel matrix $\K\in\R^{n\times n}$ with entries $K_{ij}=\kappa(\x_i,\x_j)$. The optimal partitioning of $\X$ into $m$ clusters is denoted by $\SSS^{opt}$:
	\begin{eqnarray}
	\SSS^{opt}=\argmin_{\SSS}E\left(\X,\SSS\right),\;\text{where}\;E\left(\X,\SSS\right)=\sum_{i=1}^{n}\|\x_i-\mu(\x_i)\|_2^2.
	\end{eqnarray}
	Let us generate a random sign matrix $\HH\in\R^{p'\times p}$ as in Equation \ref{eq:random-matrix} with $p'=\order(m/\varepsilon^2)$ for some parameter  $\varepsilon\in(0,1/3)$. The Randomized Clustered Nystr\"om method computes the product $\widehat{\X}=\HH\X$ to generate a set of $m$ landmark points $\Z=[\z_1,\ldots,\z_m]$ by partitioning of $\widehat{\X}\in\R^{p'\times n}$ into $m$ clusters.
	We assume that the partitioning $\widehat{\SSS}^{opt}$ of $\widehat{\X}$ leads to $E(\widehat{\X},\widehat{\SSS}^{opt})$  
	within a constant factor of the optimal value, cf.\ \citep{kmeans_plusplus,ostrovsky2012effectiveness}. 
	Given matrices $\CC\in\R^{n\times m}$ and $\WW\in\R^{m\times m}$ whose entries are $C_{ij}=\kappa(\x_i,\z_j)$ and $W_{ij}=\kappa(\z_i,\z_j)$, the Nystr\"om approximation error is bounded with probability at least $0.96$ over the randomness of $\HH$:
	\begin{eqnarray}
	\mathcal{E}\defeq\|\K-\CC\WW^\dagger\CC^T\|_F\leq\eta_1\sqrt{(2+\varepsilon)E\left(\X,\SSS^{opt}\right)}+ \eta_2(2+\varepsilon)E\left(\X,\SSS^{opt}\right),
	\end{eqnarray}
	where $\eta_1$ and $\eta_2$ are two positive constants.
\end{theorem}
\begin{proof}
	Based on Proposition~\ref{thm:clusteredNys}, we get the following approximation error for the Randomized Clustered Nystr\"om method:
	\begin{eqnarray}
	\mathcal{E}\leq\eta_1\sqrt{E(\X,\widehat{\SSS}^{opt})}+\eta_2E(\X,\widehat{\SSS}^{opt}),\label{eq:Random-proof-step1}
	\end{eqnarray}
	where $\widehat{\SSS}^{opt}$ is the optimal partitioning of the reduced data $\widehat{\X}$ and $E(\X,\widehat{\SSS}^{opt})$ represents the total quantization error when $\widehat{\SSS}^{opt}$ is used 
	to cluster the high-dimensional data $\X$. 
	We assume the partitioning in the reduced data set is within a constant factor of optimal, so this constant is absorbed into $\eta_1$ and $\eta_2$.
	In~\citep{Randomized_Dim_K_means}, it is shown that by choosing $p'=\order(m/\varepsilon^2)$ dimensions for the random projection matrix $\HH$, the following inequality holds with probability at least $0.96$ over the randomness of $\HH$:  
	\begin{eqnarray}
	E(\X,\widehat{\SSS}^{opt})\leq(2+\varepsilon)E(\X,\SSS^{opt}).
	\end{eqnarray}

	Thus, employing the above inequality in Equation \ref{eq:Random-proof-step1} completes the proof. 
\end{proof}

The error bound in Theorem~\ref{thm:randomized-clustered-nys} reveals important insights about the performance of our proposed method. Although our Randomized Clustered Nystr\"om generates landmark points based on the random projections of data, we can relate the approximation error to the total quantization error of partitioning the \emph{original} data points. In fact, our results show that the random projections of original data points into $\R^{p'}$ with $p'=\order(m)$ yields an approximation which is close to the one obtained by the Clustered Nystr\"om method (Proposition~\ref{thm:clusteredNys}). Interestingly, the dimension of reduced data $p'$ is independent of the ambient dimension $p$ and depends only on $m$ (the number of landmark points) and $\varepsilon$ (the distortion factor). As a result, for high-dimensional data sets with large $p$, the dimension of reduced data $p'$ can be fixed based on the desired number of landmark points and accuracy.
\section{Numerical Experiments}\label{sec:experiments}
In this section, we present experimental results comparing our Randomized Clustered Nystr\"om with a few other sampling methods such as the Clustered Nystr\"om method and uniform sampling. Our proposed approach is implemented in MATLAB with the  C/mex implementation for computing the sample mean in step~\ref{alg:pass2} of Algorithm \ref{alg:RandomizedClustered}. 
To perform the K-means clustering algorithm, we use MATLAB's built-in function \texttt{kmeans} and the maximum number of iterations is set to $10$. This function utilizes the K-means++ algorithm~\citep{kmeans_plusplus} for cluster center initialization which improves the performance over random initializations. Note that the K-means++ algorithm needs $m$ passes over the data to choose $m$ initial cluster centers. Since our Randomized Clustered Nystr\"om performs K-means clustering on the low-dimensional random projections  (step~\ref{alg:k-means-low} of Algorithm \ref{alg:RandomizedClustered}), the overall number of passes on the data set will remain two.

The performance of our proposed Randomized Clustered Nystr\"om method is demonstrated on two different tasks: low-rank approximation of kernel matrices (Section~\ref{sec:exp-quality}) and kernel ridge regression (Section~\ref{sec:exp-ridge}). All the experiments are conducted on a desktop computer with two Intel Xeon EF-2650 v3 CPUs at 2.4--3.2 GHz and 8 cores.

\subsection{Kernel Approximation Quality}\label{sec:exp-quality}
In this section, we study the accuracy and efficiency of our Randomized Clustered Nystr\"om on the low-rank approximation of kernel matrices in the form of $\K\approx\LL\LL^T$, where $\LL\in\R^{n\times r}$ for target rank $r$. Experiments are conducted on four data sets from the LIBSVM archive~\citep{CC01a}, listed in Table~\ref{table:data}. In all experiments, similar to~\citep{zhang2010clusteredNys}, the Gaussian kernel $\kappa\left(\x_i,\x_j\right)=\exp\left(-\|\x_i-\x_j\|_2^2/c\right)$ is used with the parameter $c$ chosen as the averaged squared distances between all the data points and sample mean.
The approximation accuracy is measured by the normalized kernel approximation error in terms of the Frobenius norm: $\|\K-\LL\LL^T\|_F/\|\K\|_F$.  We report the mean and standard deviation of approximation error over $50$ trials because all sampling methods in the Nystr\"om method involve some randomness. 
\begin{table}[ht]
	\centering
	\begin{tabular}{SSS} \toprule
		{data set} & {$p$} & {$n$} \\ \midrule
		\dataset{dna} & {180} & {$2,\!000$}\\
		\dataset{protein} & {357} & {$17,\!766$}\\
		\dataset{mnist} & {784} & {$60,\!000$}\\
		\dataset{epsilon} & {$2,\!000$} & {$60,\!000$}\\ \bottomrule
	\end{tabular}
	\caption{Summary of data sets used in Section~\ref{sec:exp-quality}.}
		\label{table:data}
\end{table}

For two data sets \dataset{dna} and \dataset{protein} with the total number of data points less than $20,\!000$, the entire kernel matrix $\K\in\R^{n\times n}$ can be stored and manipulated in the main memory of the computer. Thus,  
the accuracy of our proposed method for various values of the compression factor $\gamma$ is compared with:
\begin{enumerate}[leftmargin=*]
	\item The SVD, where the best rank-$r$ approximation of the kernel matrix $\K$ is obtained via the exact eigenvalue decomposition or SVD; 
	\item Uniform sampling~\citep{Nystrom2001}, where $m$ landmark points are selected uniformly at random without replacement from $n$ data points;
	\item Column-norm sampling~\citep{drineas2006fast}, where $m$ columns of the kernel matrix $\K$ are sampled with weights proportional to $\ell_2$ norm of columns;
	\item Clustered Nystr\"om~\citep{zhang2010clusteredNys}, where  $m$ landmark points are generated using centroids resulting from K-means clustering on the original data set.
\end{enumerate} 
Also, for all choices of landmark points (Randomized Clustered Nystr\"om and options 2, 3, 4 above) with $m$ greater than the target rank $r$, our proposed ``Nystr\"om via QR decomposition'' (Algorithm \ref{alg:NysQR}) is used to restrict the resulting Nystr\"om approximation to have rank at most $r$, cf.~Section~\ref{sec:nys-eig}.

For the two large-scale data sets \dataset{mnist} and \dataset{epsilon}, it takes approximately $29$ GB space to store the kernel matrix. Thus, the storage and manipulation of $\K$ in the main memory becomes too costly and our proposed method is compared with just the two sampling methods which do not need access to the entire kernel matrix, namely uniform sampling and Clustered Nystr\"om. 
\subsubsection{Data Sets: \dataset{dna} and \dataset{protein}}
The first example demonstrates the effectiveness of various sampling methods on improving the accuracy of the Nystr\"om method by increasing the number of landmark points.
The mean and standard deviation of kernel approximation error are reported in Figure \ref{fig:fixed-rank-dna-protein} for varying number of landmark points $m$ with fixed target rank $r=3$. The results in Figure \ref{fig:fixed-rank-dna-1} and Figure \ref{fig:fixed-rank-protein-1} show that both Clustered Nystr\"om and our proposed method with $\gamma=0.02$ ($p'=4$ for \dataset{dna} and $p'=7$ for \dataset{protein}) improve the accuracy of the Nystr\"om method over uniform sampling and column-norm sampling. 
In fact, the accuracy of our proposed method and Clustered Nystr\"om reaches the accuracy of the best rank-$r$ approximation (SVD) for small values of $m$, e.g., $m=r$.  The uniform sampling method does not reach this accuracy even if it uses a large number of landmark points such as  $m=10r=30$.

To further investigate the tradeoffs between accuracy and efficiency of our proposed method, the mean and standard deviation of kernel approximation error for a few values of the compression factor $\gamma$ from $0.01$ to $0.2$ are presented in Figure \ref{fig:fixed-rank-dna-2} and Figure \ref{fig:fixed-rank-protein-2} (error bars for $\gamma=0.01,0.05,0.2$ have been omitted for clarity). As the compression factor $\gamma$ (equivalently, $p'$) increases, the approximation error decreases which is consistent with our theoretical results in Theorem~\ref{thm:randomized-clustered-nys}. However, small values of $\gamma$ in our method, such as $\gamma=0.01$, lead to accurate low-rank approximations with savings in memory and computation by a factor of $1/\gamma=100$. As a final note, it is observed that our method with $\gamma=0.2$ performs slightly better than the Clustered Nystr\"om method on the \dataset{protein} data set. This is mainly due to the fact that the performance of K-means clustering depends on the starting points. It is possible for K-means to reach a local minimum solution, where a better solution with the lower value of objective function exists. In practice, one can increase the number of random initializations and select the clustering with the lowest value of objective function. The difference in performance between Clustered Nystr\"om and our method with $\gamma=0.2$ disappears if we instead take the best result out of 20 independent initializations.

\begin{figure}[t]
	\begin{centering}
		\subfloat[Kernel approximation error, \dataset{dna}]{\begin{centering}
				\label{fig:fixed-rank-dna-1}
				\includegraphics[scale=0.4]{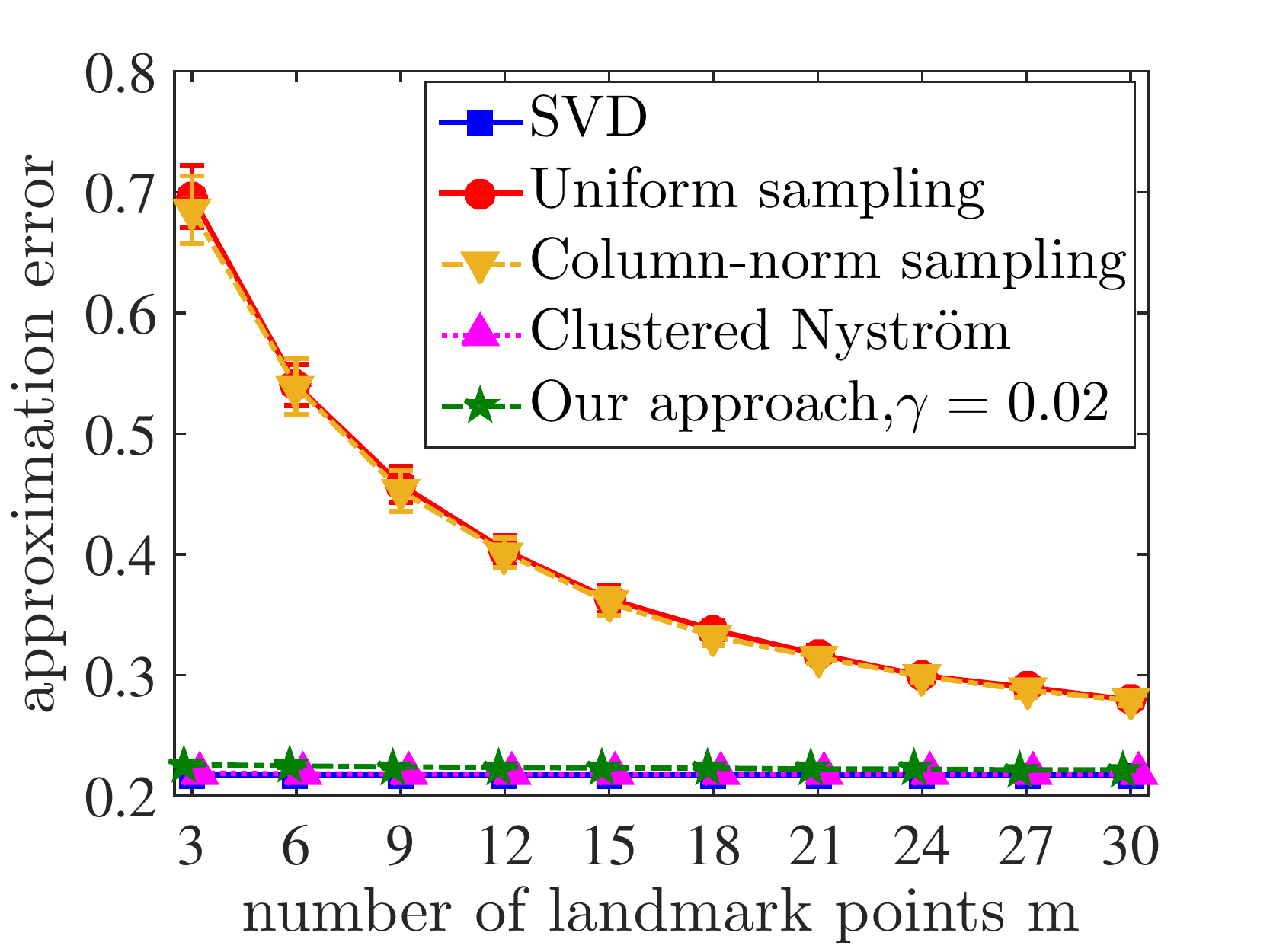}
				\par\end{centering}
		}\subfloat[\dataset{dna} (zoom)]{\begin{centering}
			\label{fig:fixed-rank-dna-2}
			\includegraphics[scale=0.40]{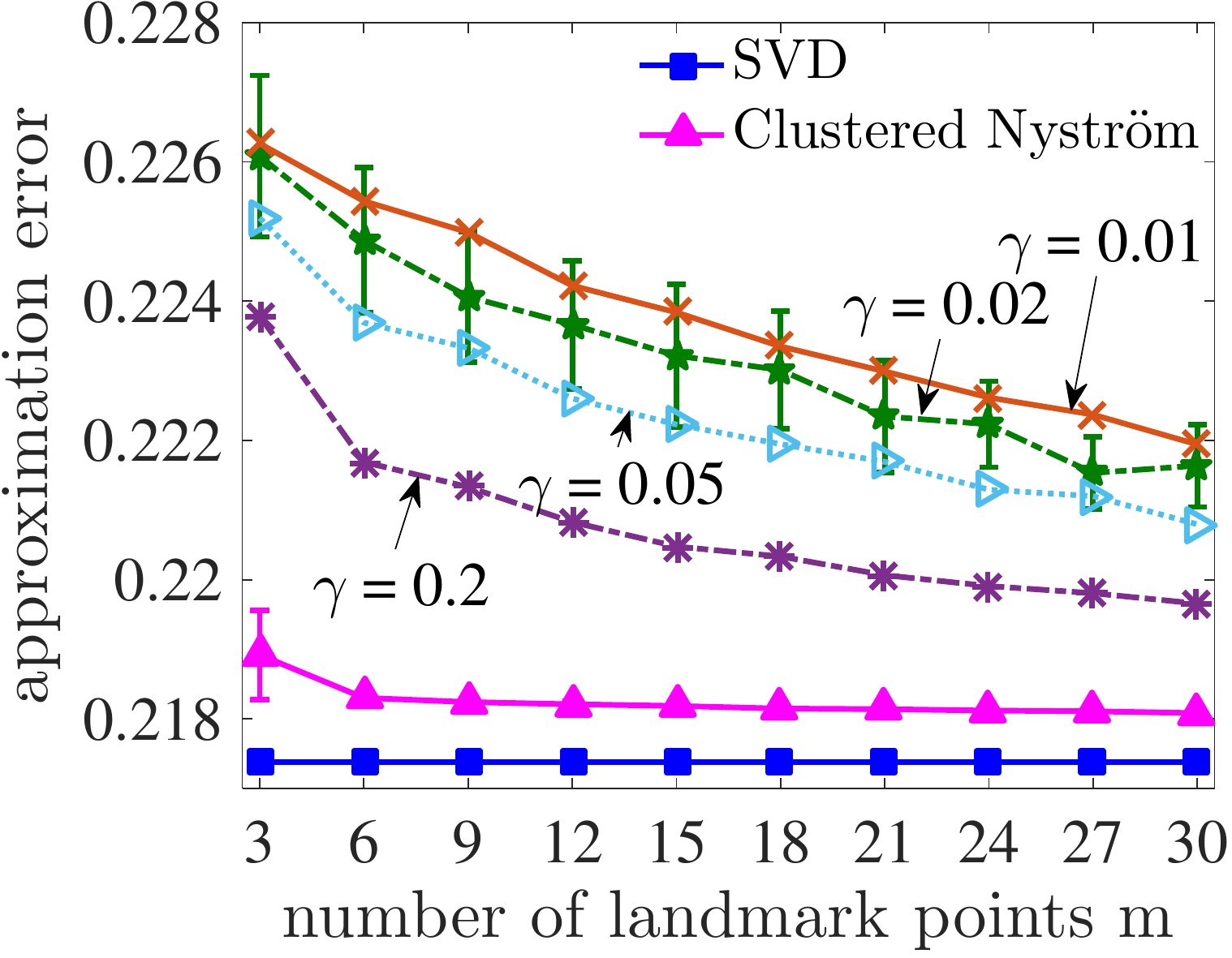}
			\par\end{centering}
	}
	\par\end{centering}

\begin{centering}
	\subfloat[Kernel approximation error, \dataset{protein}]{\begin{centering}
			\label{fig:fixed-rank-protein-1}
			\includegraphics[scale=0.40]{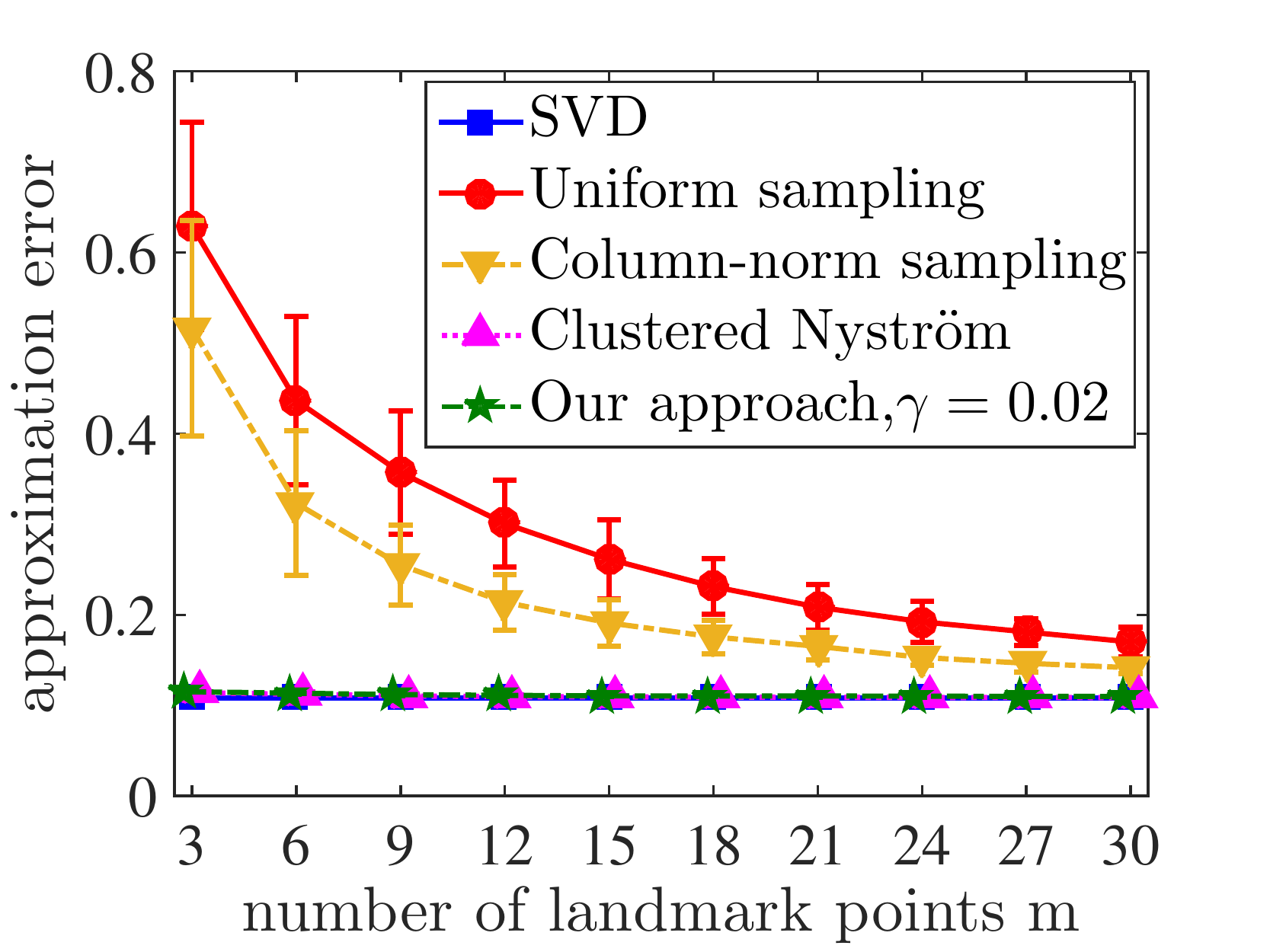}
			\par\end{centering}
	}\subfloat[\dataset{protein} (zoom)]{\begin{centering}
		\label{fig:fixed-rank-protein-2}
		\includegraphics[scale=0.40]{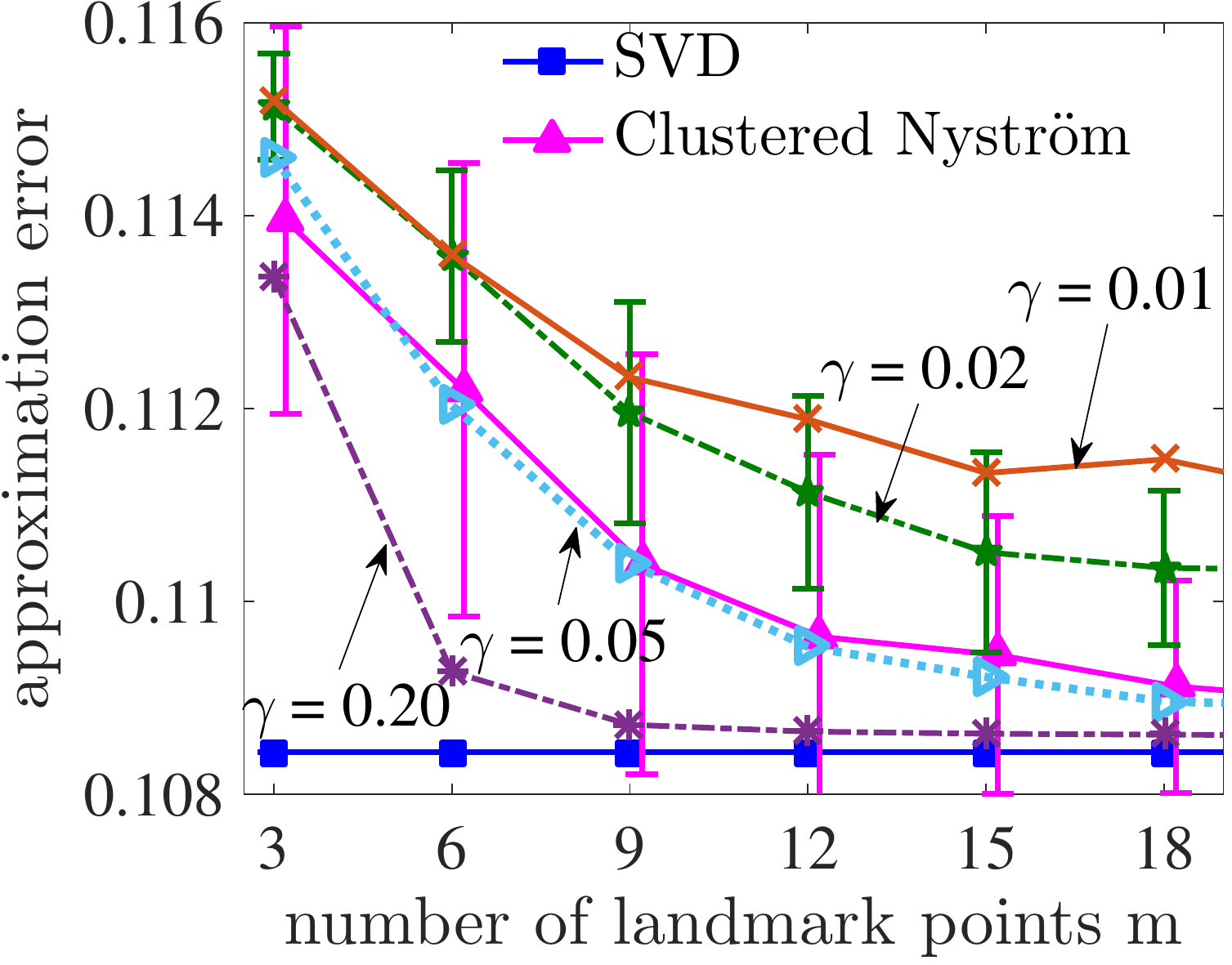}
		\par\end{centering}
}
\par\end{centering}
\caption{Kernel approximation error for varying number of landmark points $m$ with target rank $r=3$ on \dataset{dna} and \dataset{protein} data sets.}
\label{fig:fixed-rank-dna-protein}
\end{figure}

\begin{figure}[ht]
	\begin{centering}
		\subfloat[Kernel approximation error, \dataset{dna}]{\begin{centering}
				\label{fig:low-rank-dna-1}
				\includegraphics[scale=0.40]{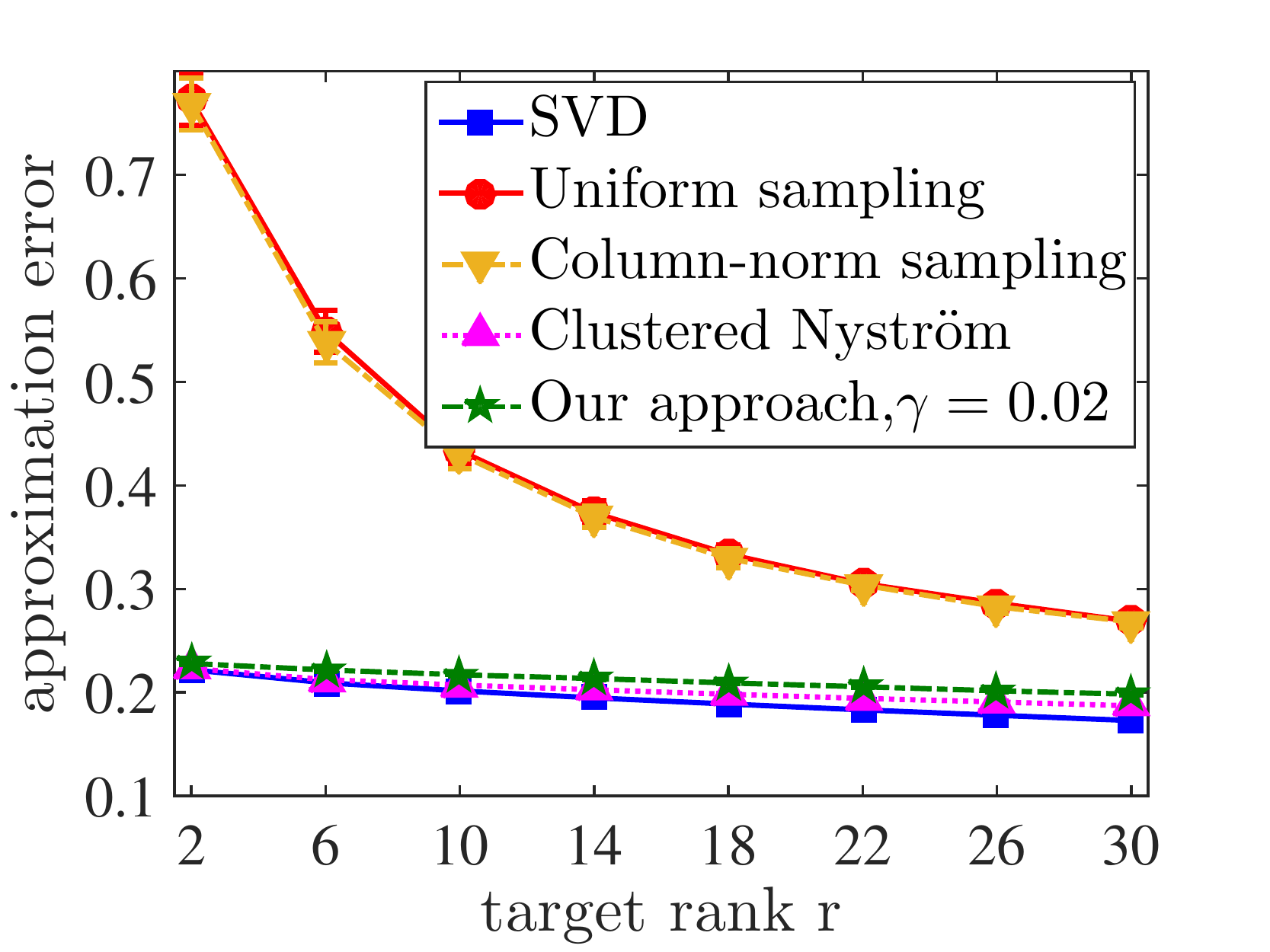}
				\par\end{centering}
		}\subfloat[\dataset{dna} (zoom)]{\begin{centering}
			\label{fig:low-rank-dna-2}
			\includegraphics[scale=0.40]{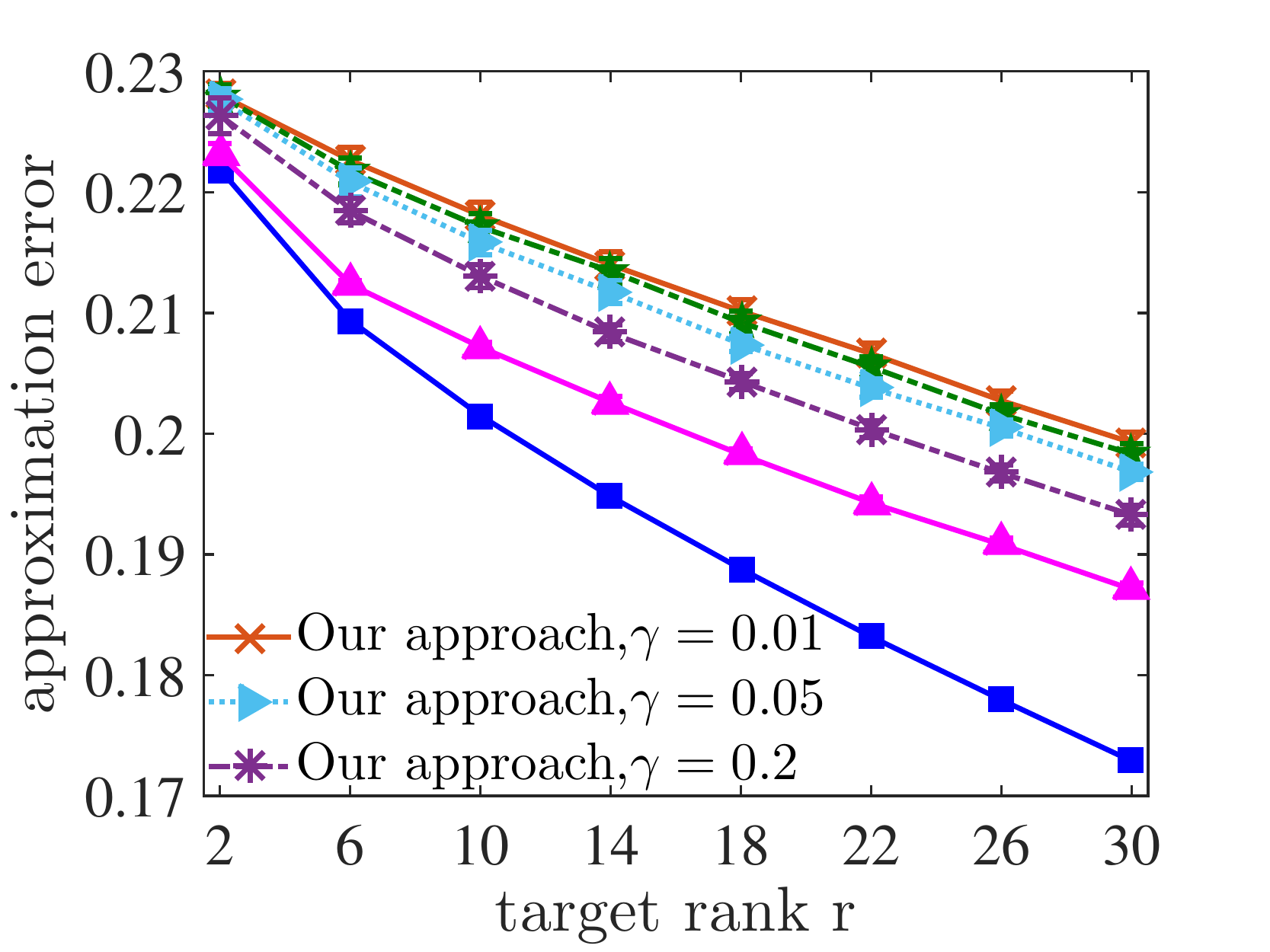}
			\par\end{centering}
	}
	\par\end{centering}

\begin{centering}
	\subfloat[Kernel approximation error, \dataset{protein}]{\begin{centering}
			\label{fig:low-rank-protein-1}
			\includegraphics[scale=0.40]{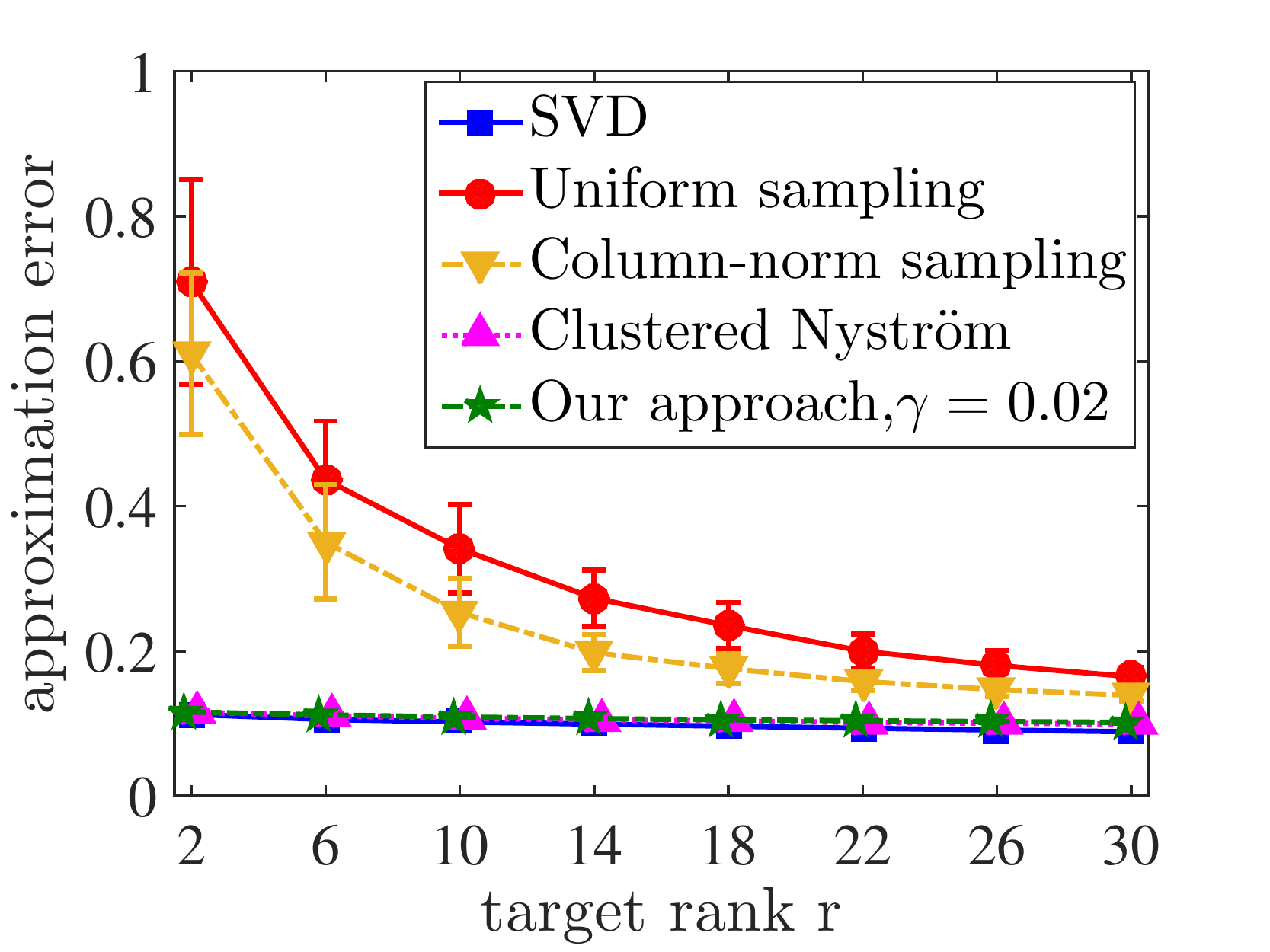}
			\par\end{centering}
	}\subfloat[\dataset{protein} (zoom)]{\begin{centering}
		\label{fig:low-rank-protein-2}
		\includegraphics[scale=0.40]{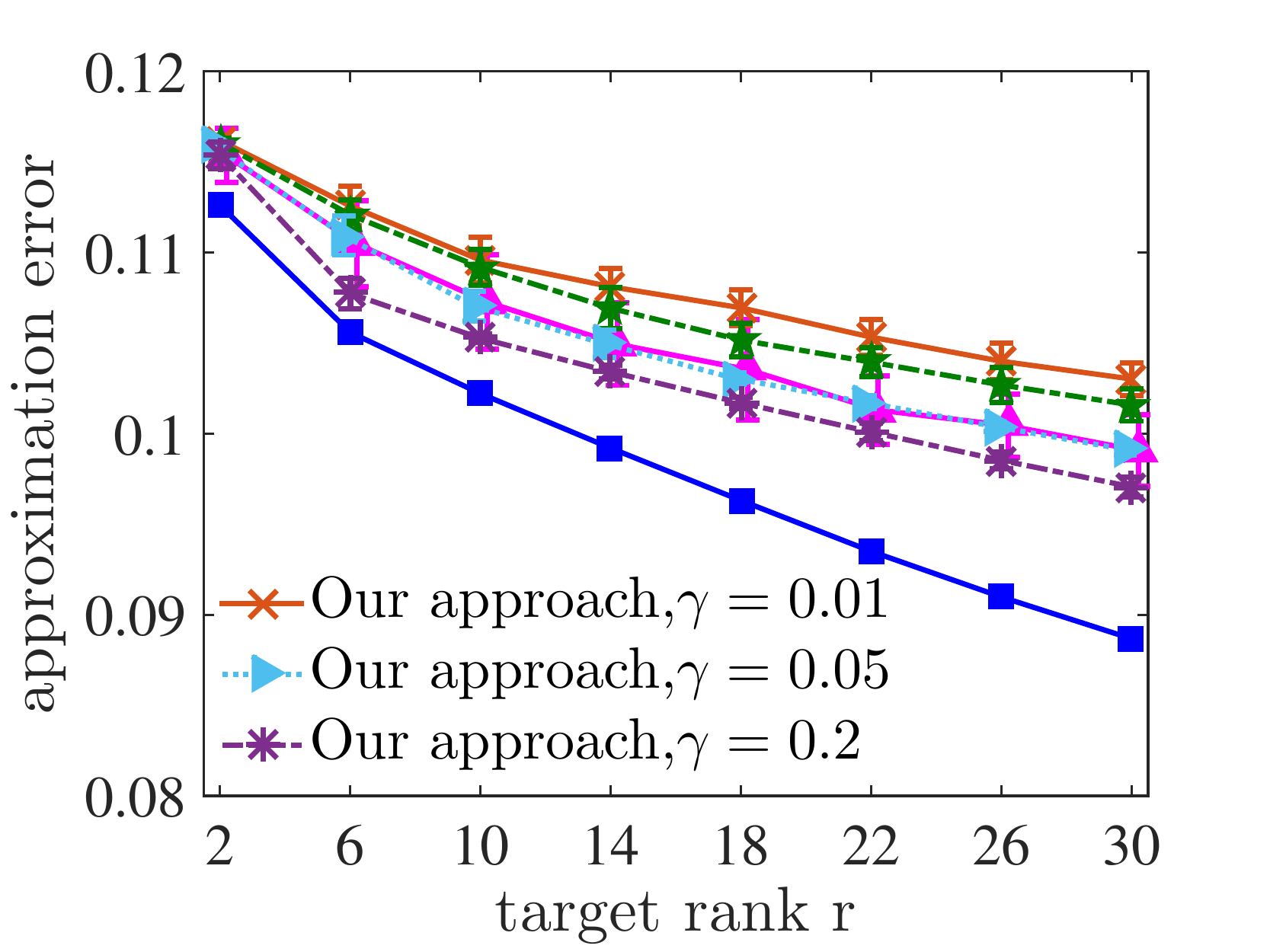}
		\par\end{centering}
}
\par\end{centering}
\caption{Kernel approximation error for various values of target rank $r$ and fixed $m=r$ on the \dataset{dna} and \dataset{protein} data sets.}
\label{fig:low-rank-dna-protein}
\end{figure}

The second example, shown in Figure \ref{fig:low-rank-dna-protein}, demonstrates the performance of our proposed method for various values of target rank $r$ from $2$ to $30$ and fixed $m=r$. Based on Figure \ref{fig:low-rank-dna-1} and Figure \ref{fig:low-rank-protein-1}, it is clear that both our method and Clustered Nystr\"om method provide improved approximation accuracy over uniform sampling and column-norm sampling techniques. In fact, we see that both our method and Clustered Nystr\"om method have roughly the same accuracy as the best rank-$r$ approximation (SVD) for all values of the target rank.

We also report the mean and standard deviation of kernel approximation error in Figure \ref{fig:low-rank-dna-2} and Figure \ref{fig:low-rank-protein-2} for varying values of compression factor $\gamma$ from $0.01$ to $0.2$. 
As the compression factor $\gamma$ (or the number of dimensions $p'$) increases, the kernel approximation error decreases as prescribed by our theoretical results in Theorem~\ref{thm:randomized-clustered-nys}. Also, we see that small values of compression factor $\gamma$ result in accurate low-rank approximations which lead to memory and computation savings by a factor of $1/\gamma$ in comparison with the Clustered Nystr\"om method.

\subsubsection{Data Sets: \dataset{mnist} and \dataset{epsilon}}
The accuracy and time complexity of our Randomized Clustered Nystr\"om method are demonstrated on two large-scale examples.  The parameter $\gamma$ is set to $0.01$ for the \dataset{mnist} data set ($p'=8$) and $0.005$ for the \dataset{epsilon} data set ($p'=10$). 

In the first example, the normalized kernel approximation error and computation time are reported in Figure \ref{fig:fixed-rank-largescale} for various values of $m$ and fixed target rank $r=3$.
In Figure \ref{fig:fixed-rank-mnist-error} and Figure \ref{fig:fixed-rank-epsilon-error}, we observe that our proposed method outperforms uniform sampling and has almost the same accuracy as the Clustered Nystr\"om method for all values of $m$. However, the runtime of our proposed method is reduced by an order of magnitude compared to the Clustered Nystr\"om method (Figure \ref{fig:fixed-rank-mnist-time} and Figure \ref{fig:fixed-rank-epsilon-time}). Thus, our proposed method provides significant memory and computation savings with little loss in accuracy compared to the Clustered Nystr\"om method. While our randomized method spends more time than uniform sampling to find a small set of informative landmark points, it provides improved approximation accuracy. Thus, these empirical results suggest a tradeoff between time and space requirements of our proposed method and uniform sampling. For example, our method with $m=r$ landmark points outperforms uniform sampling with $m=10r$ on the \dataset{epsilon} data set. 
\begin{figure}[t]
	\begin{centering}
		\subfloat[Kernel approximation error, \dataset{mnist}]{\begin{centering}
				\label{fig:fixed-rank-mnist-error}
				\includegraphics[scale=0.40]{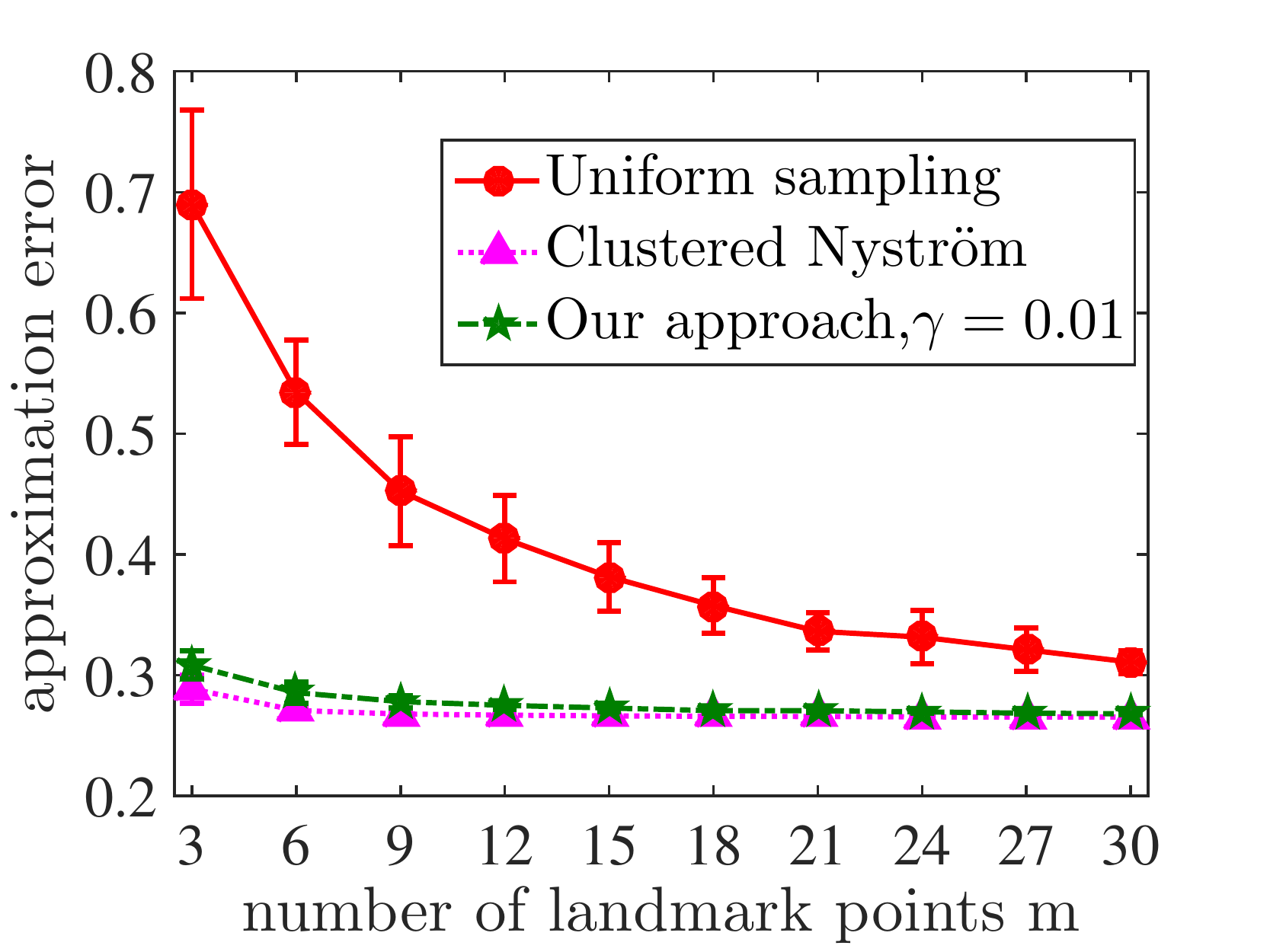}
				\par\end{centering}
		}\subfloat[Runtime, \dataset{mnist}]{\begin{centering}
			\label{fig:fixed-rank-mnist-time}
			\includegraphics[scale=0.40]{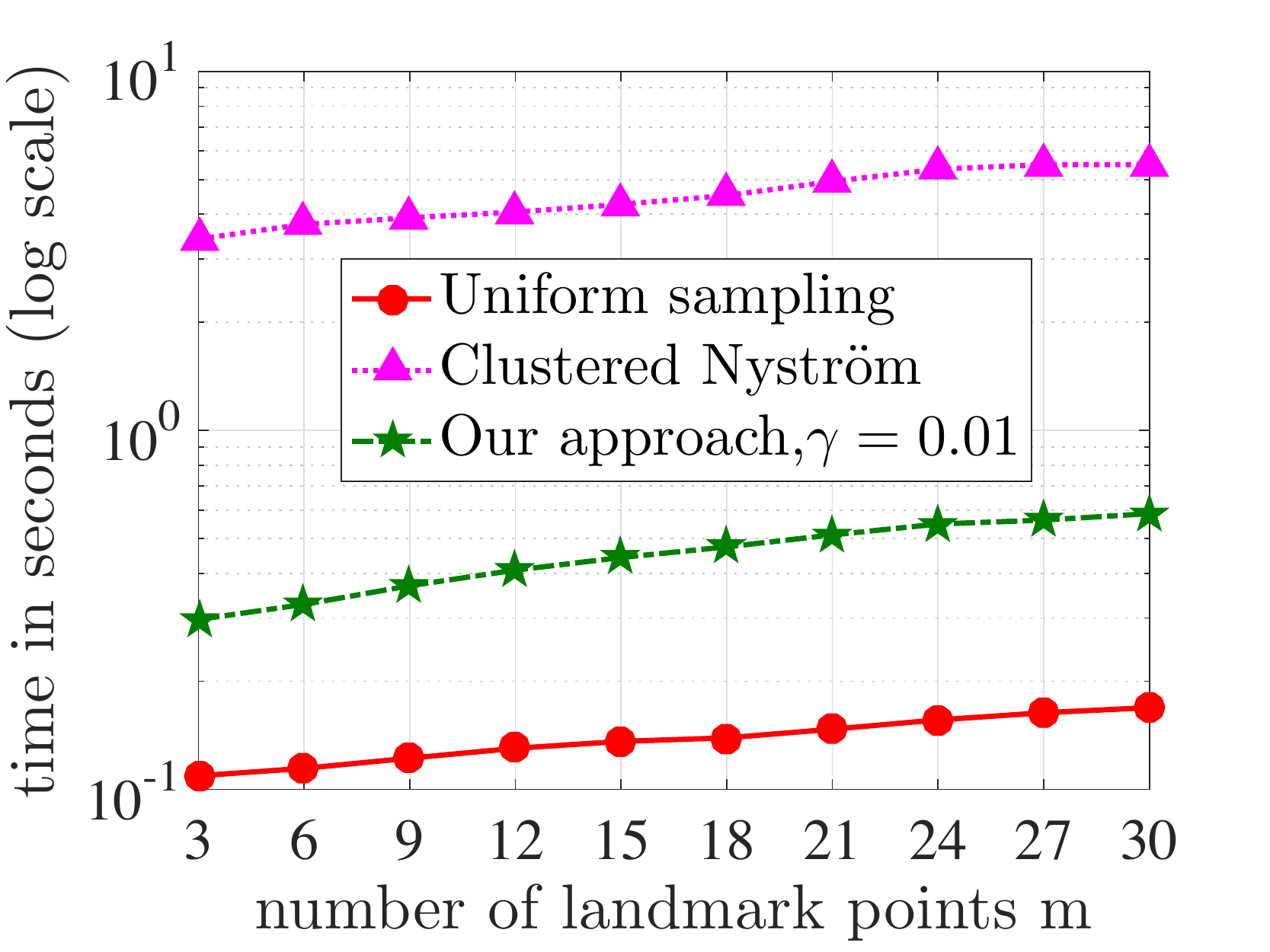}
			\par\end{centering}
		
	}
	\par\end{centering}
\begin{centering}
	\subfloat[Kernel approximation error, \dataset{epsilon}]{\begin{centering}
			\label{fig:fixed-rank-epsilon-error}
			\includegraphics[scale=0.40]{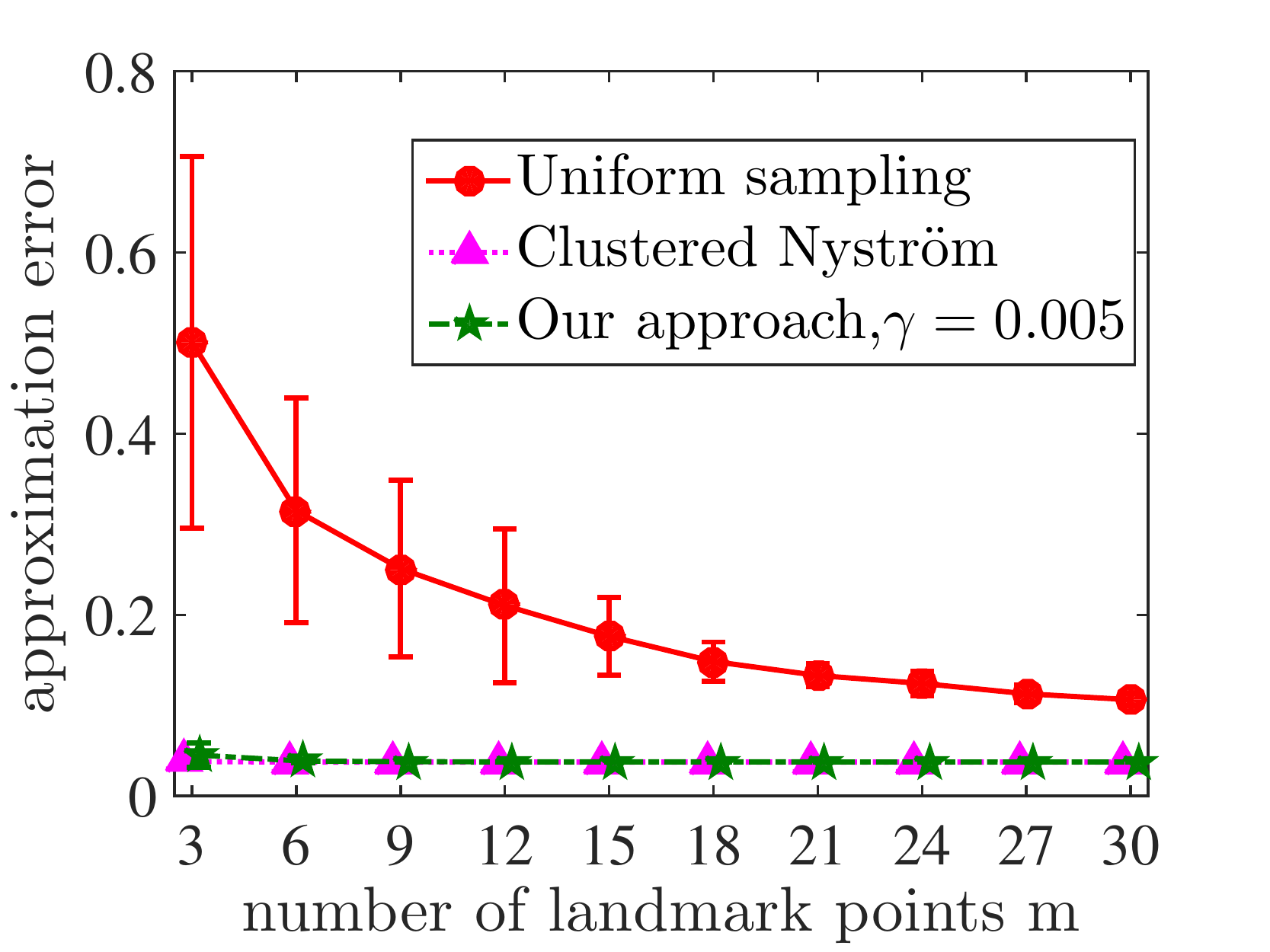}
			\par\end{centering}
	}\subfloat[Runtime, \dataset{epsilon}]{\begin{centering}
		\label{fig:fixed-rank-epsilon-time}
		\includegraphics[scale=0.40]{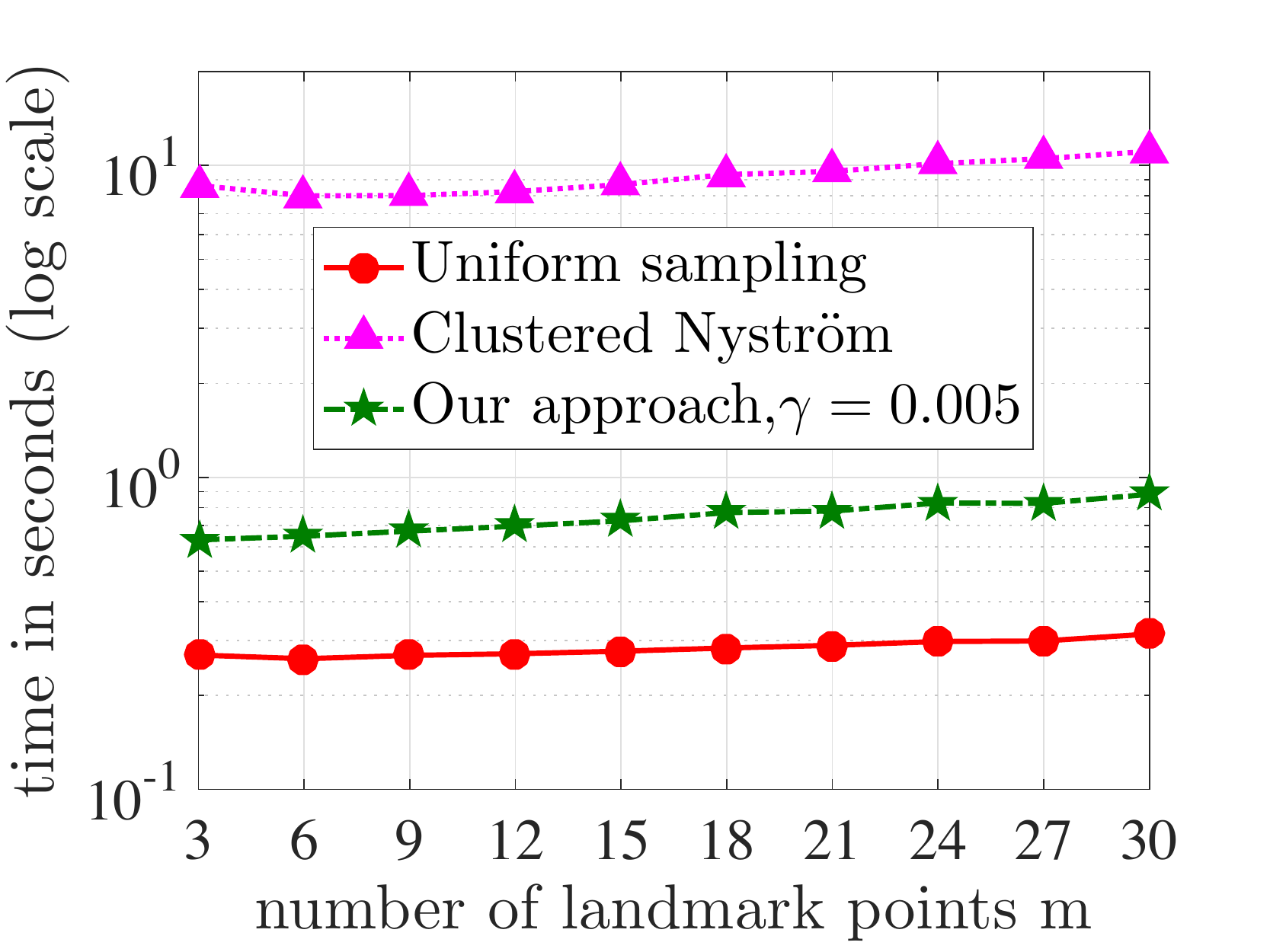}
		\par\end{centering}
}
\par\end{centering}
\caption{Kernel approximation error and runtime for varying number of landmark points $m$ with target rank $r=3$ on \dataset{mnist} and \dataset{epsilon} data sets.}
\label{fig:fixed-rank-largescale}
\end{figure}

\begin{figure}[t]
	\begin{centering}
		\subfloat[Kernel approximation error, \dataset{mnist}]{\begin{centering}
				\label{fig:low-rank-mnist-error}
				\includegraphics[scale=0.40]{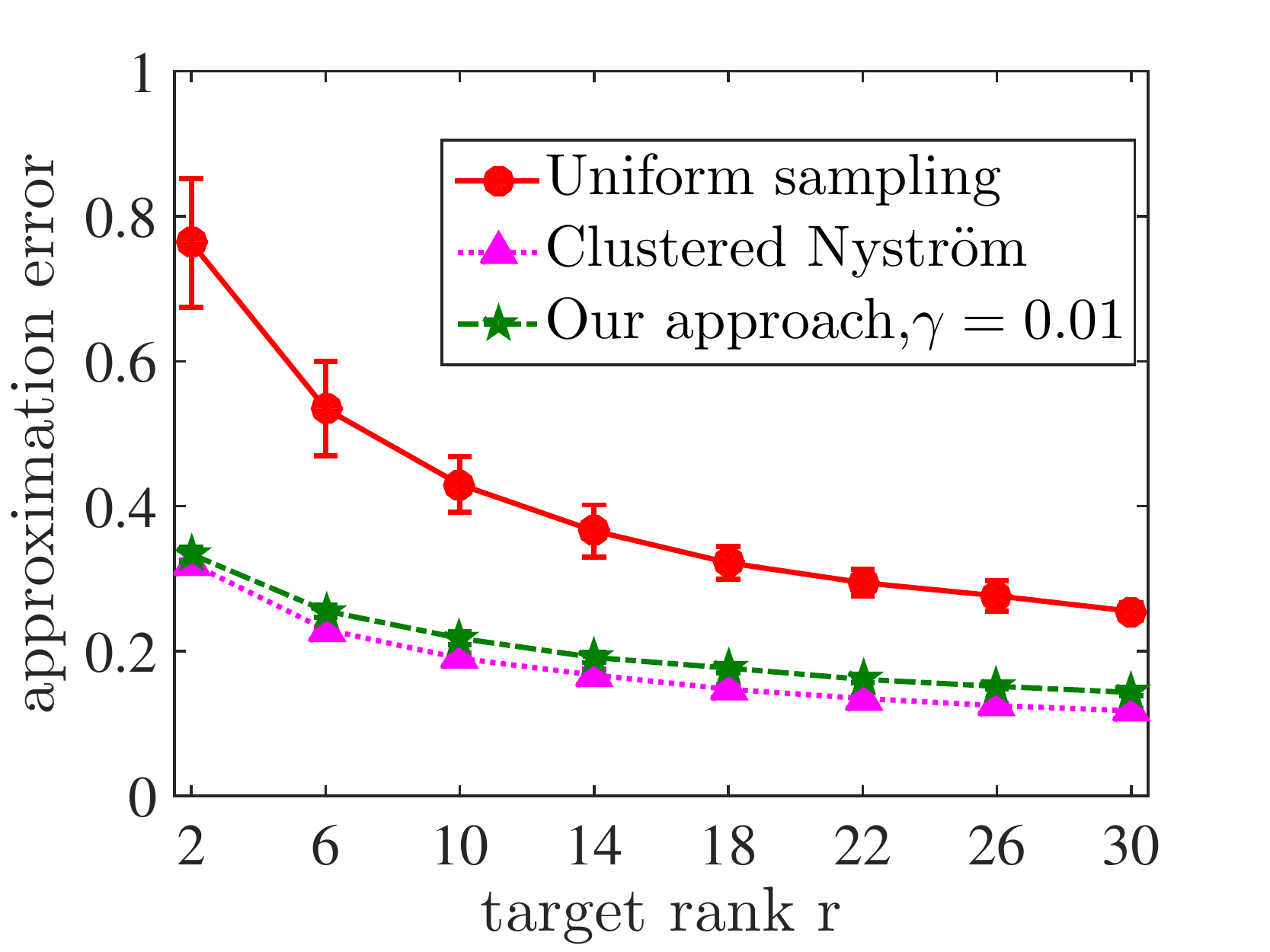}
				\par\end{centering}
		}\subfloat[Runtime, \dataset{mnist}]{\begin{centering}
			\label{fig:low-rank-mnist-time}
			\includegraphics[scale=0.40]{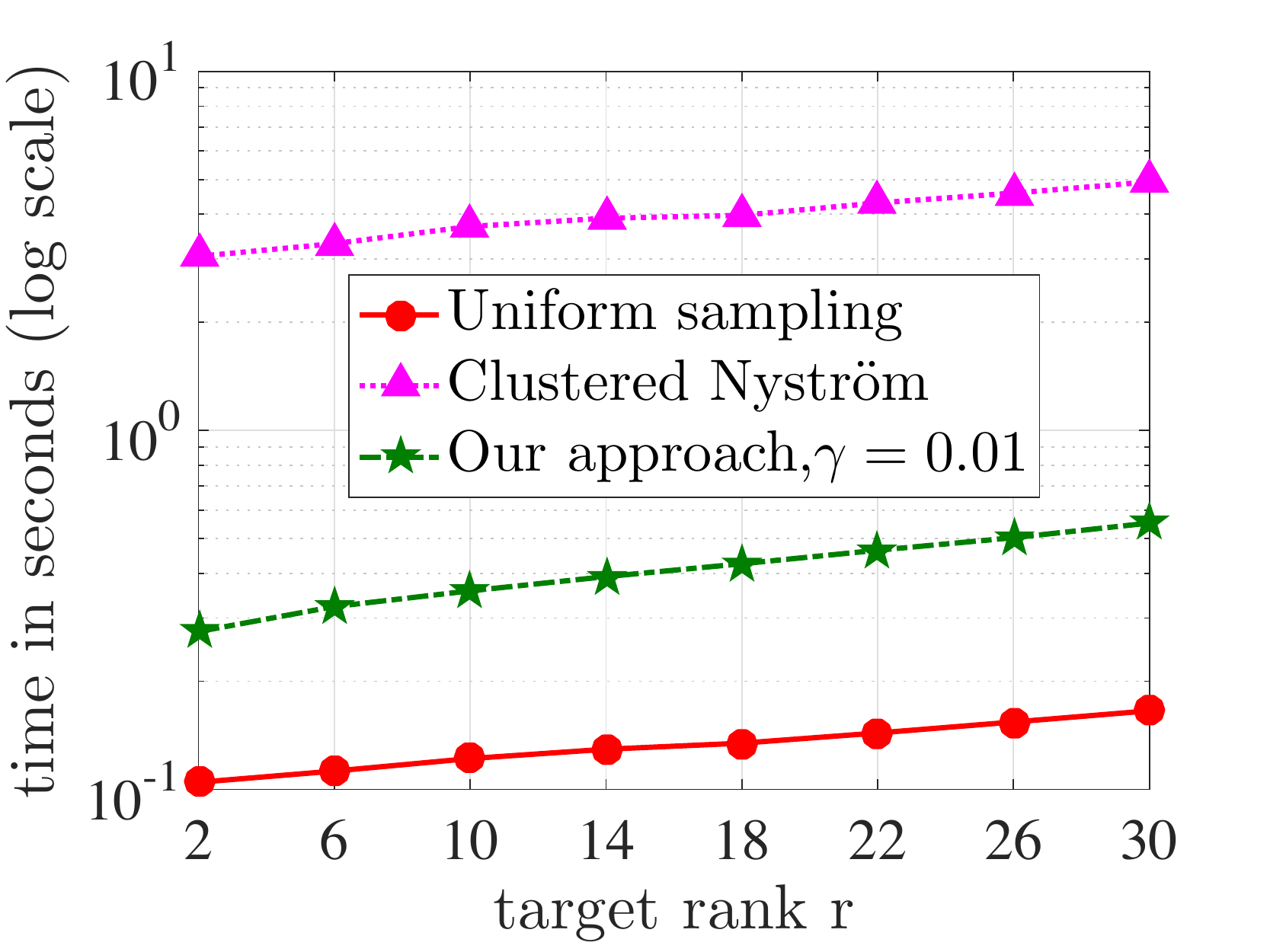}
			\par\end{centering}
	}
	\par\end{centering}
\begin{centering}
	\subfloat[Kernel approximation error, \dataset{epsilon}]{\begin{centering}
			\label{fig:low-rank-epsilon-error}
			\includegraphics[scale=0.40]{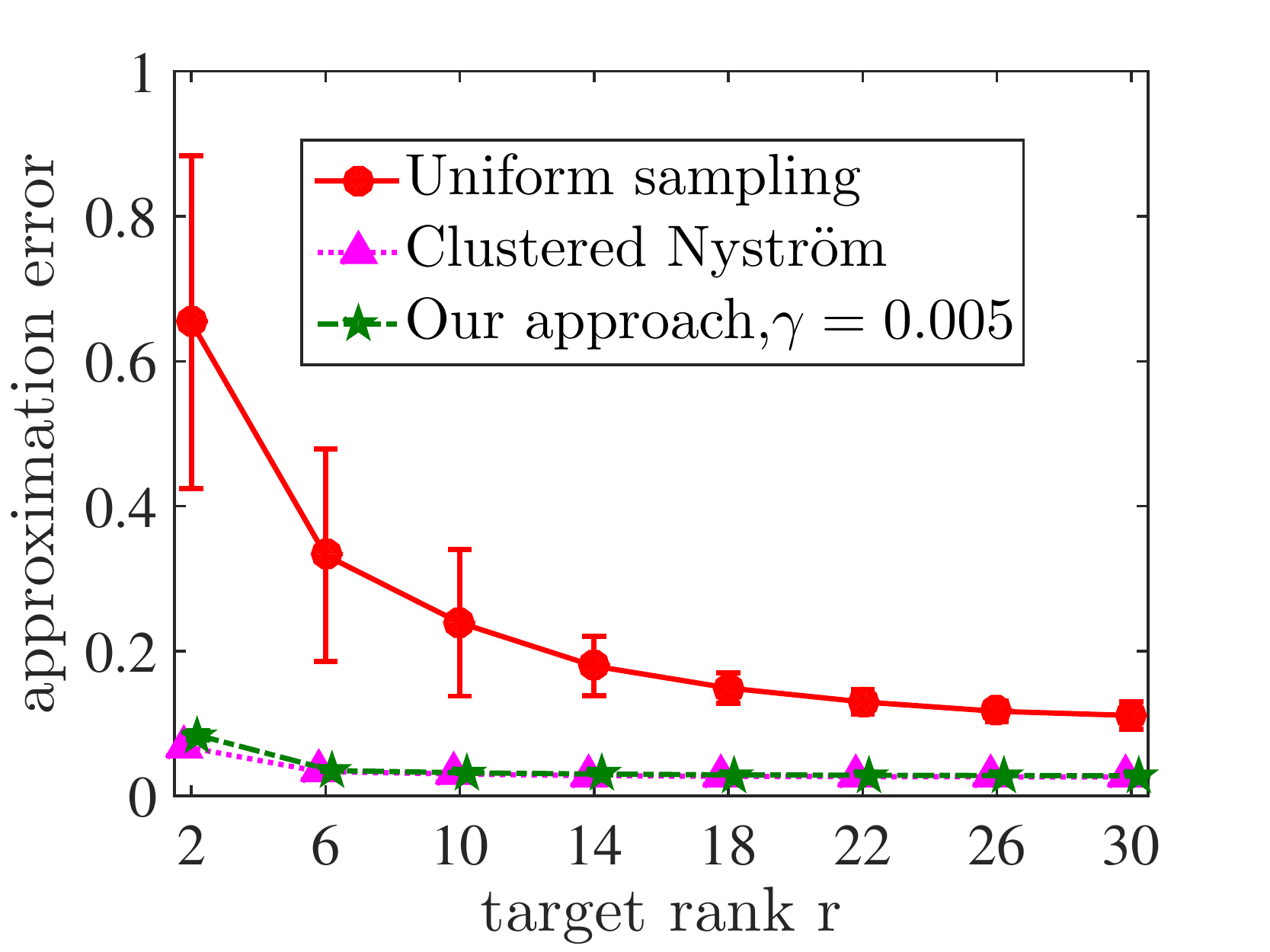}
			\par\end{centering}
	}\subfloat[Runtime, \dataset{epsilon}]{\begin{centering}
		\label{fig:low-rank-epsilon-time}
		\includegraphics[scale=0.40]{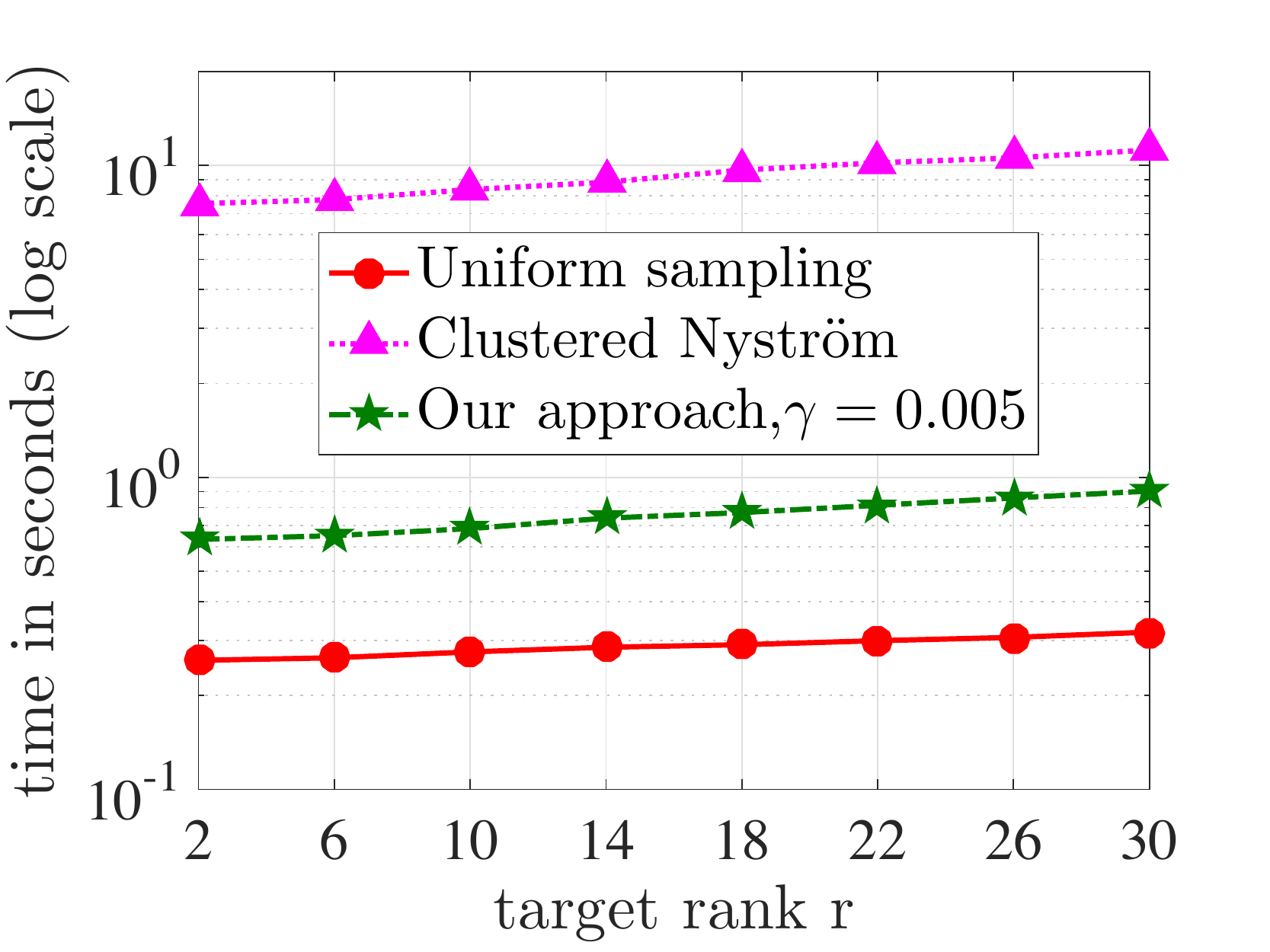}
		\par\end{centering}
}
\par\end{centering}	
\caption{Kernel approximation error and runtime for various values of target rank $r$ and fixed $m=r$ on \dataset{mnist} and \dataset{epsilon} data sets.}
\label{fig:low-rank-largescale}
\end{figure}

In the second example, our proposed method is compared with uniform sampling and Clustered Nystr\"om for various values of the target rank $r$ from $2$ to $30$ (fixing $m=r$) and results are reported in Figure \ref{fig:low-rank-largescale}. As we see in Figure \ref{fig:low-rank-mnist-error} and Figure \ref{fig:low-rank-epsilon-error}, our proposed method outperforms uniform sampling for all values of the target rank and has almost the same accuracy as the Clustered Nystr\"om method. Similar to the previous example, the time complexity of our proposed method is decreased by an order of magnitude compared to the Clustered Nystr\"om. Hence, our Randomized Clustered Nystr\"om provides improved approximation accuracy, while being more efficient in terms of memory and computation than the Clustered Nystr\"om method.
\subsection{Kernel Ridge Regression}\label{sec:exp-ridge}
In this section, we present experimental results on the performance of various sampling methods when used with kernel ridge regression. In the supervised learning setting, a set of instance-label pairs $\{(\x_i,y_i)\}_{i=1}^{n}$ are given, where $\x_i\in\R^p$ and $y_i\in\R$. Kernel ridge regression proceeds by generating $\boldsymbol{\alpha}^*$ which solves the dual optimization problem~\citep{saunders1998ridge}:
\begin{eqnarray}
\min_{\boldsymbol{\alpha}\in\R^n}\;\boldsymbol{\alpha}^T\K\boldsymbol{\alpha}+\lambda\boldsymbol{\alpha}^T\boldsymbol{\alpha}-2\boldsymbol{\alpha}^T\y,
\end{eqnarray}
where $\K$ is the kernel matrix, $\y=[y_1,\ldots,y_n]^T\in\R^n$ is the response vector, and $\lambda>0$ is the regularization parameter. The problem admits the closed-form solution $\boldsymbol{\alpha}^*=(\K+\lambda\eye_{n\times n})^{-1}\y$, which requires the computation of $\K\in\R^{n\times n}$ and the $n\times n$ system to solve. Memory and computation cost can be reduced by using the low-rank approximation of the kernel matrix $\K\approx\LL\LL^T$, where $\LL\in\R^{n\times r}$,  to generate an approximate solution (cf.~Equation \ref{eq:ridge-eq}):
\begin{eqnarray}
\widehat{\boldsymbol{\alpha}}=\lambda^{-1}\left(\eye_{n\times n}-\LL\big(\LL^T\LL+\lambda\eye_{n\times n}\big)^{-1}\LL^T\right).
\end{eqnarray}
\citet{cortes2010impact} analyzed the effect of such low-rank approximations on the accuracy of the approximate solution $\widehat{\boldsymbol{\alpha}}$. 

Here, we empirically compare the accuracy of our Randomized Clustered Nystr\"om with a few other sampling methods. The approximation error is defined as $\|\widehat{\boldsymbol{\alpha}}-\boldsymbol{\alpha}^*\|_2/\|\boldsymbol{\alpha}^*\|_2$ and we report the mean and standard deviation of the approximation error over $50$ trials. Two data sets from the LIBSVM archive~\citep{CC01a} are considered for regression: (1) \dataset{cpusmall} and (2) \dataset{E2006-tfidf}. The former data set consists of $n=8,\!192$ samples with $p=12$ and we increase the dimensionality to $p=48$ by repeating each entry $4$ times. The \dataset{E2006-tfidf} data set contains $n=5,\!363$ samples with $p=150,\!360$. As before, the Gaussian kernel function $\kappa\left(\x_i,\x_j\right)=\exp\left(-\|\x_i-\x_j\|_2^2/c\right)$ is used with the parameter $c$ chosen as the averaged squared distances as in the previous section. The regularization parameter $\lambda$ is set to $1/4$.
\begin{figure}[t]
	\begin{centering}
		\subfloat[Approximation error of $\boldsymbol{\alpha}^*$, $r=82$]{\begin{centering}
				\label{fig:reg_cpu_1}
				\includegraphics[scale=0.40]{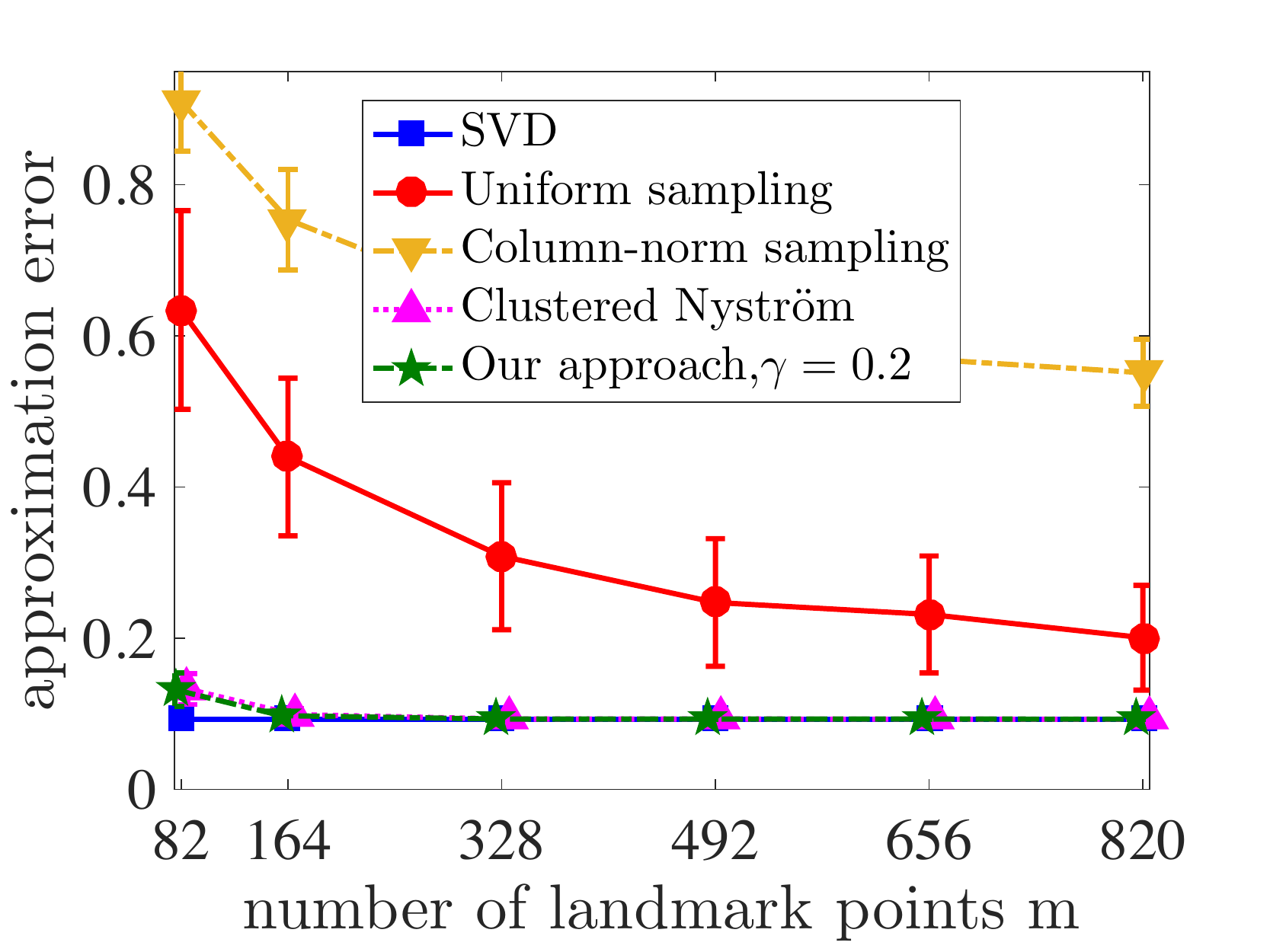}
				\par\end{centering}
		}\subfloat[Approximation error of $\boldsymbol{\alpha}^*$, $m=2r$]{\begin{centering}
			\label{fig:reg_cpu_2}
			\includegraphics[scale=0.40]{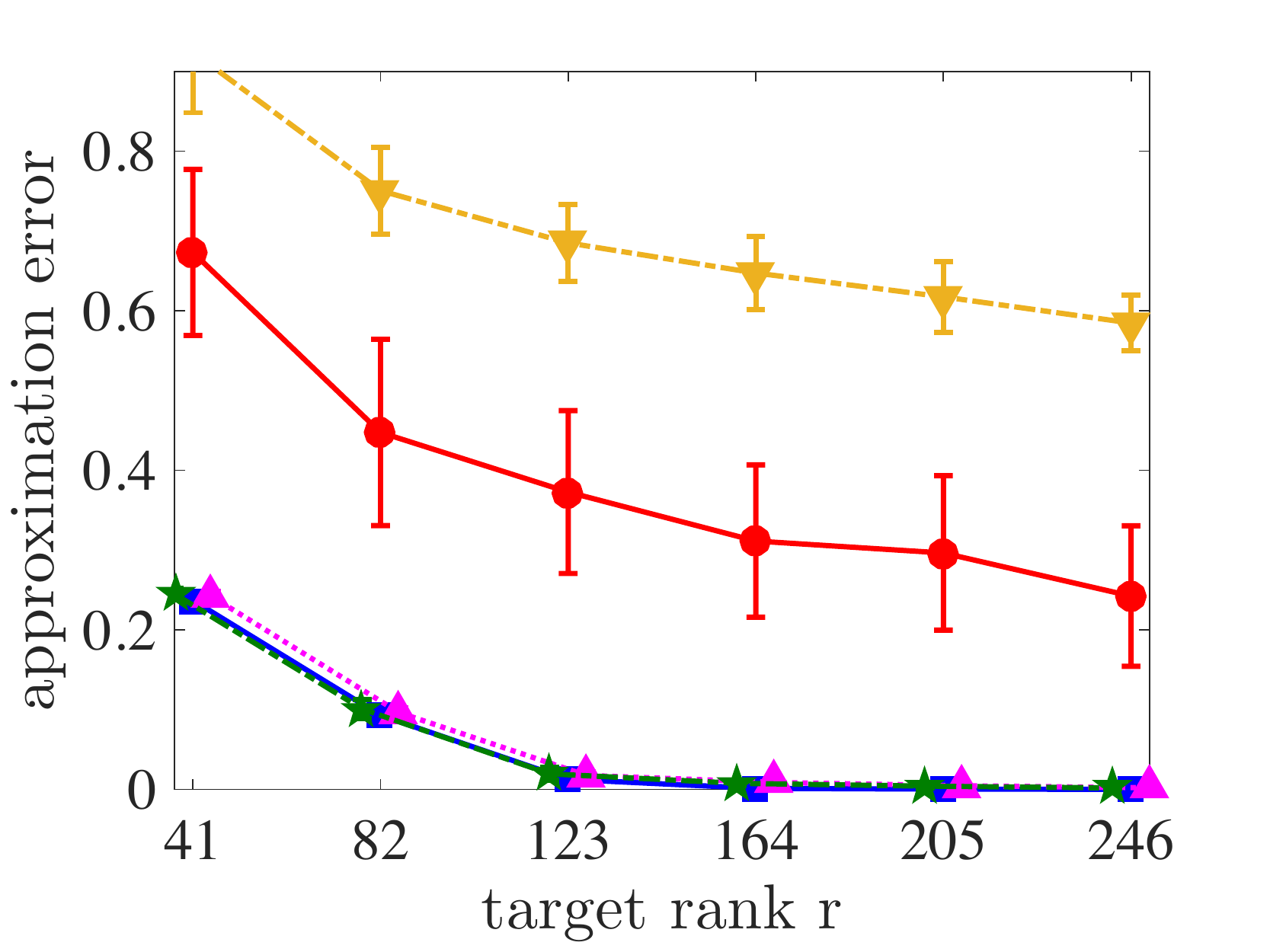}
			\par\end{centering}
	}
	\par\end{centering}
\caption{Kernel ridge regression on the \dataset{cpusmall} data set.}
\label{fig:regression-cpusmall}
\end{figure}

In Figure \ref{fig:regression-cpusmall}, the approximation error on the \dataset{cpusmall} data set is reported when $\gamma=0.2$ in our method ($p'=10$) for two cases: (1) fixed target rank $r/n=0.01$ and varying number of landmark points $m$ (Figure \ref{fig:reg_cpu_1}); (2) various values of the target rank $r$ from $0.005n$ to $0.03n$ and fixed $m=2r$ (Figure \ref{fig:reg_cpu_2}). The performance of our Randomized Clustered Nystr\"om is significantly better than that of the uniform sampling and column-norm sampling approaches. In fact, our proposed method is as nearly accurate as the best rank-$r$ approximation (SVD) for just $m=2r$. 
There is no significant difference in performance between our method and the full Clustered Nystr\"om method.

In Figure \ref{fig:regression-E2006}, the approximation error on the \dataset{E2006-tfidf} data set is presented for the target rank $r/n=0.03$ and varying number of landmark points. The parameter $\gamma$ in our Randomized Clustered Nystr\"om is set to $6.5\times 10^{-5}$ which means that $p'=10$. 
Our method is much more accurate than both uniform and column-norm sampling as shown in  Figure \ref{fig:reg_E2006_1}. Moreover, based on Figure \ref{fig:reg_E2006_2}, our method reduces the computational complexity of the Clustered Nystr\"om method by two orders of magnitude
since Clustered Nystr\"om performs K-means on a
very high-dimensional data set.

\begin{figure}[t]
	\begin{centering}
		\subfloat[Approximation error of $\boldsymbol{\alpha}^*$, $r=161$]{\begin{centering}
				\label{fig:reg_E2006_1}
				\includegraphics[scale=0.40]{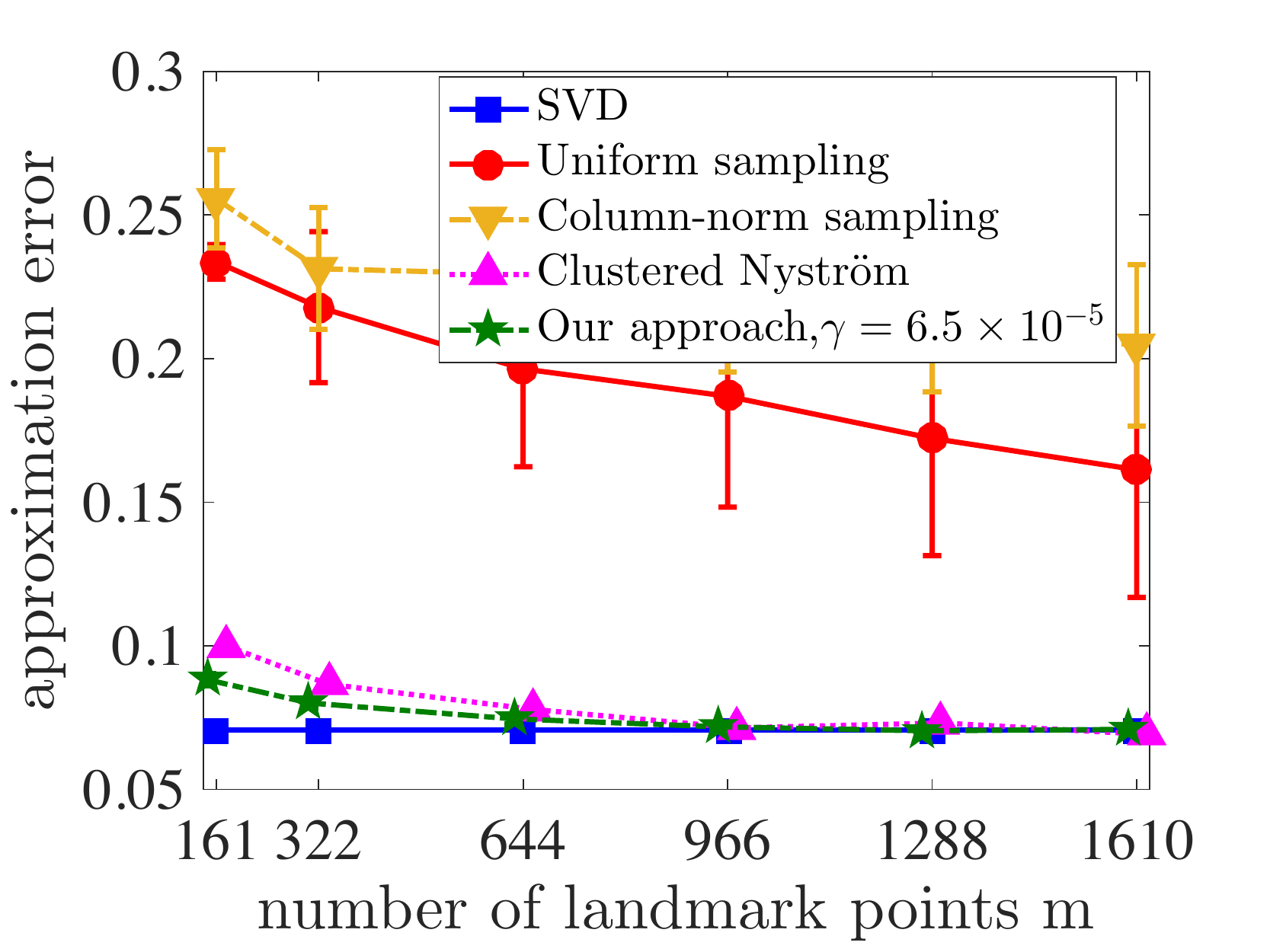}
				\par\end{centering}
		}\subfloat[Runtime, $r=161$]{\begin{centering}
			\label{fig:reg_E2006_2}
			\includegraphics[scale=0.40]{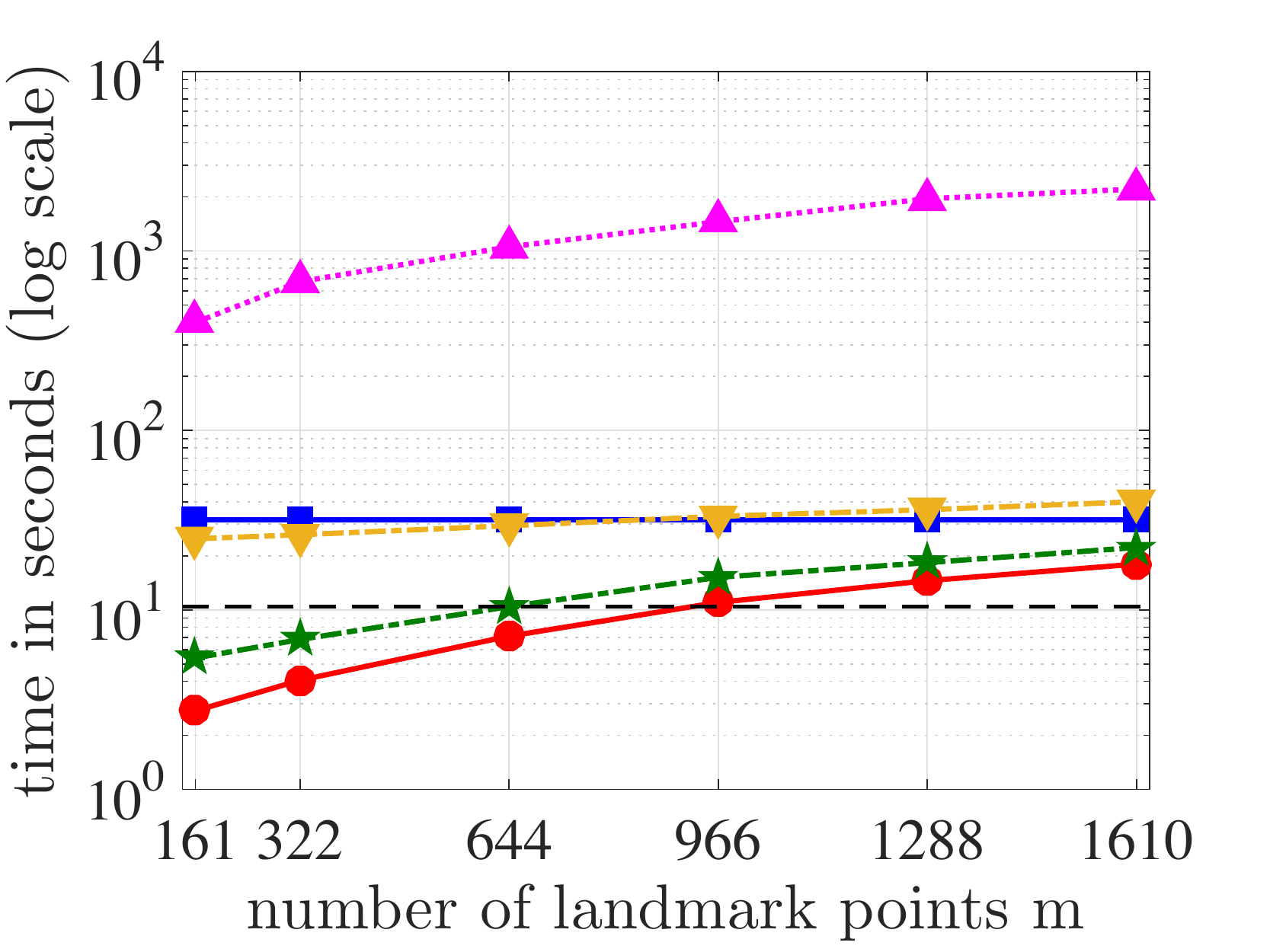}
			\par\end{centering}
	}
	\par\end{centering}
\caption{Kernel ridge regression on the \dataset{E2006-tfidf} data set.}
\label{fig:regression-E2006}
\end{figure}

We draw a dashed line in Figure \ref{fig:reg_E2006_2} to find the values of $m$ for which our method and uniform sampling have the same running time. We see that $m=644$ in our method, which leads to mean error $0.074$ and standard deviation $0.001$, has the same running time as $m=966$ in the uniform sampling, with mean $0.187$ and standard deviation $0.038$. This is an example of our method performing both more accurately and more efficiently than the alternatives.

\section{Conclusion}
In this paper, we presented two complementary methods to improve the quality of Nystr\"om low-rank approximations. The first method, ``Nystr\"om via QR Decomposition,'' finds the best rank-$r$ approximation when the number of landmark points $m$ is greater than the target rank $r$ in the Nystr\"om method. The experimental examples demonstrated the superior performance of our proposed method compared to the standard Nystr\"om method. Also, for a fixed accuracy, the introduced method requires fewer landmark points which is of great importance for the efficiency of the Nystr\"om method. The second proposed method, ``Randomized Clustered Nystr\"om,'' is a randomized algorithm for generating landmark points, where the memory and computational complexity can be adjusted by using the compression factor $\gamma$. Based on our experiments, random projection of the input data points onto a low-dimensional space with $p'=10$ or smaller dimensions yields very accurate low-rank approximations. In fact, the accuracy of our proposed method is very close to the best rank-$r$ approximation obtained by the exact eigenvalue decomposition or SVD of kernel matrices.

\section*{Acknowledgments}
This work is supported by the Bloomberg Data Science Research Grant Program.
\vskip 0.2in
\bibliography{phd_farhad.bib}

\end{document}